\documentclass[sigconf, nonacm]{acmart}
\usepackage{booktabs} 
\usepackage{color} 
\usepackage{balance} 
\usepackage{url} 
\usepackage{siunitx} 
\usepackage{epsfig,tabularx,subfigure,multirow, graphicx}
\usepackage{amsthm,amsmath}  
\usepackage{enumitem}

\usepackage{warning} 

\usepackage{hyperref} 

\usepackage[nolist,printonlyused,withpage]{acronym} 

\let\oldbibliography\thebibliography \renewcommand{\thebibliography}[1]{%
	\oldbibliography{#1}%
	\setlength{\itemsep}{0pt}%
}

\usepackage[linesnumbered,ruled,vlined]{algorithm2e}
\SetKwRepeat{Do}{do}{while}
\SetCommentSty{mycommfont}

\long\def\comment#1{}

\newcommand{\nop}[1]{}

\newtheorem{theorem}{\bf Theorem}[section]
\newtheorem{lemma}{\bf Lemma}[section]

\newtheorem{example}{\bf Example}

\theoremstyle{remark}

\theoremstyle{definition}
\newtheorem{definition}{\bf Definition}

\newcommand\vldbdoi{10.14778/3476249.3476271}

\newcommand\vldbavailabilityurl{http://vldb.org/pvldb/format_vol14.html}
\newcommand\vldbpagestyle{empty} 
\newcommand{\revision}[1]{\color{black}{#1} \color{black}}

\begin{document}
	\begin{acronym}
	\acro{API}{Application Programming Interface}
\end{acronym}

\title{A Queueing-Theoretic Framework for Vehicle Dispatching in Dynamic Car-Hailing [technical report]}

\author{Peng Cheng}
\orcid{0000-0002-9797-6944}
\affiliation{%
  \institution{East China Normal University}
  \city{Shanghai}
  \country{China}
}
\email{pcheng@sei.ecnu.edu.cn}

\author{Jiabao Jin}
\affiliation{%
  \institution{East China Normal University}
  \city{Shanghai}
  \state{China}
}
\email{10175101146@stu.ecnu.edu.cn}

\author{Lei Chen}
\orcid{0000-0001-5109-3700}
\affiliation{%
  \institution{The Hong Kong University of Science and Technology}
  \city{Hong Kong}
  \country{China}
}
\email{leichen@cse.ust.hk}

\author{Xuemin Lin}
\affiliation{%
	\institution{The University of New South Wales}
	\city{Sydney}
	\country{Australia}
}
\email{lxue@cse.unsw.edu.au}

\author{Libin Zheng}
\affiliation{%
	\institution{Guangdong Key Laboratory of Big Data Analysis and Processing, Sun Yat-sen University}
	\city{Guangzhou}
	\country{China}
}
\email{zhenglb6@mail.sysu.edu.cn}


\begin{abstract}
With the rapid development of smart mobile devices, the car-hailing platforms (e.g., Uber or Lyft) have attracted much attention from the academia and the industry. In this paper, we  consider a dynamic car-hailing problem, namely \textit{maximum revenue vehicle dispatching} (MRVD), in which rider requests dynamically arrive and drivers need to serve  riders  such that the entire revenue of the platform is maximized. We prove that the MRVD problem is NP-hard and intractable. To handle the MRVD problem, we propose a queueing-based vehicle dispatching framework, which first uses existing machine learning models to predict the future vehicle demand of each region, then estimates the idle time periods of drivers through a double-sided queueing model for each region. With the information of the predicted vehicle demands and estimated idle time periods of drivers, we propose two batch-based vehicle dispatching algorithms to efficiently assign suitable drivers to riders such that the expected overall revenue of the platform is maximized during each batch processing. Through extensive experiments, we demonstrate the efficiency and effectiveness of our proposed approaches over both real and synthetic datasets. In summary, our methods can achieve $3\%\sim10\%$ increase on overall revenue without sacrificing on running speed compared with the state-of-the-art solutions.
\end{abstract}

\maketitle

\pagestyle{\vldbpagestyle}

\ifdefempty{\vldbavailabilityurl}{}{
\vspace{.3cm}
}

\section{Introduction}

Recently, with the popularity of the smart devices and high quality of the 
wireless networks, people can easily access network and communicate with 
online services. With the convenient car-hailing platforms (e.g., Uber \cite{uber} and DiDi Chuxing \cite{didi}), drivers can share their vehicles to riders to obtain monetary benefits and alleviate the pressure of public transportation. One of the crucial issues in the platforms is to efficiently dispatch vehicles to suitable riders. Although the platforms  become huge recently, during peak hours (e.g., 8 am) in some high demand areas (e.g., residential areas), riders need to wait for up to several hours before being served. To mitigate the shortage of vehicles in particular time and areas and improve the efficiency of the platforms, we investigate a queueing-theoretic framework in this paper. 


We illustrate the general idea of our framework in the following motivation example.

\begin{figure}[t!]\centering
	\scalebox{0.4}[0.4]{\includegraphics{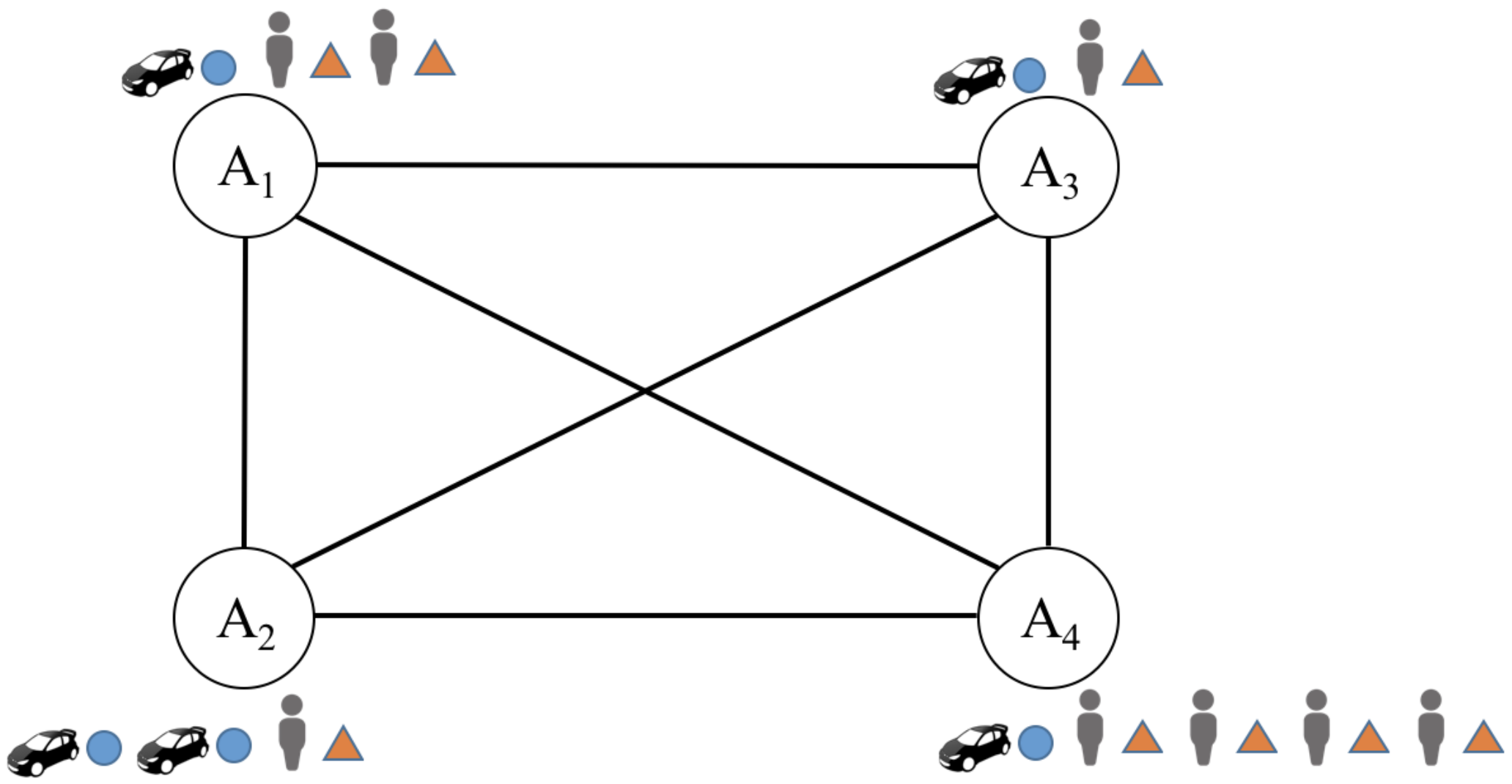}}
	\caption{\small A multi-area vehicle dispatching example.}
	\label{fig:motivation_example}\vspace{1ex}
\end{figure}

\begin{example}
\label{example1}
Consider a scenario of taxi dispatching in Figure \ref{fig:motivation_example}, where each one of the four  connected areas $A_1 \sim A_4$ maintains a queue of riders and taxis. The numbers of icons of taxis and riders near each area reflect the ratio between them in the corresponding area. For example, in area $A_1$, the number of available taxis is only half of the number of waiting riders. Riders only want to wait for a limited time (e.g., 5 minutes), otherwise they will switch to other public transportation systems (e.g., the bus system). Usually, taxis can easily pickup riders within the same area (e.g., moving several hundreds meters). However, if one taxi is assigned with a rider in a different area, the taxi may need to move several kilometers to pickup the rider. In addition, after serving its current rider, one taxi will usually wander around the destination area to pickup a new rider.

When the number of taxis is less than the demanding riders, the platform needs to smartly select riders to serve, such that the total revenue of the platform can be improved. In the example in Figure \ref{fig:motivation_example}, the platform should give higher priorities to the riders whose destination is within area $A_1$ or $A_4$, where taxis are scarce. On the contrast, the platform should give lower priorities to the riders whose destination is within area $A_2$, where taxis are abundant.
\end{example}

Motivated by the example above, in this paper, we propose a novel vehicle dispatching framework, which aims to  take riders' destinations into consideration to alleviate the shortage of taxis in particular areas such that the overall revenue of the platform can be maximized. 

Existing works in spatial matching or taxi dispatching only consider the pickup locations of riders and try to minimize the travel distance of taxis to pick riders \cite{tong2016online, seow2010collaborative}, which  causes some taxis need to wait for a long time period before picking up new convenient riders after finishing their last orders, which leads to the low efficiency of the platform (i.e., imbalanced demand-and-supply in some regions). In our previous poster paper \cite{cheng2019queueing}, to the best of our knowledge, we are the first to introduce the general idea of the queueing theoretic framework to balance the demand-and-supply of taxi dispatching during a relative long time period to maximize the overall revenue of the platform. However, in \cite{cheng2019queueing} we just explained the general queueing theoretic framework without detailed algorithms, analyses, and experimental studies, which will be introduced in this paper.


To improve the overall revenue of the platform during a relatively long time period, we propose a \textit{batch-based queueing-theoretic vehicle dispatching} framework in this paper. Specifically, we partition the whole space into regions and maintain a queue of waiting riders and drivers for each region. Once there are available drivers, the most-priority rider in the queue of waiting riders will be served. In addition, riders may \textit{quit/renege} from the platform if she is not served for a long period. Usually, a driver, after delivering her current rider, will continue to serve the next rider around the destination of the current rider. Thus, through serving the riders, the distribution of drivers will be changed. To serve as many riders as possible, intuitively, the platform should \textit{match} the distributions of riders and drivers (i.e., one region with more riders should have more drivers) through associating higher priorities to the riders whose destinations are in regions lacking of drivers, then drivers can serve riders quickly before they quit/renege. We first propose models to estimate the Poisson distributions of riders and drivers. Then, we utilize the queueing theory to analyze the idle time interval for each driver after finishing his/her assigned rider. Finally, we propose two vehicle/driver dispatching algorithms to maximize the overall revenue of the platform in each batch processing. Note that, we maximize the overall revenue of the platform through improving the efficiency of the entire platform to serve more riders without increasing the charges to riders or decreasing the payment to drivers. In fact, the payment to drivers is usually a portion of the overall revenue of the platform.  Thus, the more the overall revenue of the platform is, the more the payment to drivers is. In conclusion, our solution will benefit riders, drivers and the platform at the same time.


To summarize, we make the following contributions in the paper:
\begin{itemize}[leftmargin=*]
	\item We propose a batch-based queueing theoretic framework for vehicle dispatching in Section \ref{sec:framework}.
	\item  We estimate the idle interval time of drivers in Section \ref{sec:queue_analysis}.
	\item We propose two vehicle dispatching algorithms for each batch processing  in Section \ref{sec:task_assignment}.
	\item We have  conducted extensive experiments on real and synthetic data sets, to show the efficiency and effectiveness of our queueing-theoretic framework in Section \ref{sec:experimental}.
\end{itemize}

In addition, the remaining sections of the paper are arranged as
follows. We review and compare previous studies on queueing theory and vehicle dispatching in Section \ref{sec:related} and conclude the work in Section \ref{sec:conclusion}.

\section{Problem Definition}
\label{sec:problem_definition}

In this section, we present the formal definition of the vehicle dispatching problem, where a system will assign drivers to riders to deliver them to their destinations.

In this paper, we use a graph $G=\langle V, E\rangle$ to represent a road network, where $V$ is a set of vertices and $E$ is a set of edges. Each edge $(u, v)\in E$ ($u,v\in V$) is associated with a weight $cost(u,v)$ indicating the travel cost from vertex $u$ to vertex $v$. Here, the travel cost could be the travel time or the travel distance. When we know the travel speed of vehicles, we can convert one to another. In the rest of this paper, we will not differentiate between them and use \textit{travel cost} consistently.

To better manage the riders and drivers, we assume the entire space is divided into a set of $n$ regions/grids $A=\{a_1, a_2, ..., a_n\}$.

\subsection{Riders and Drivers}
\begin{definition}[Impatient Rider]
Let $r_i$ be an impatient rider, who submit his/her order $o_i$ to the platform at timestamp $t_i$, and is associated with a source location $s_i$, a destination location $e_i$ and a pickup deadline $\tau_i$.
\end{definition}

In particular, a rider $r_i$ comes to the platform to call for one and only one driver to deliver him/her from his/her current location $s_i$ to his/her destination location $e_i$. The request is sent to the platform at timestamp $t_i$. As the rider is impatient, if the platform cannot assign an available driver to pick up him/her within $\tau_i$ time after $t_i$, he/she will quit/renege from the platform and may switch to other platforms or use other transportation systems. Usually, rider $r_i$ will not rejoin the platform immediately after he/she is delivered to his/her destination. Thus, in this paper, we assume that each rider is unique and only lives  for the lifetime of his/her ride  in the platform. In addition, if a rider $r_i$ is delivered to his/her destination, the platform will charge him/her for $\alpha  \cdot cost(s_i, e_i)$, where $\alpha$ is the travel fee rate of the platform.

\begin{definition}[Driver]
Let $d_j$ be a driver, who is located \revision{at} position $l_j(t)$ at timestamp $t$. His/her status is either \textit{busy} (i.e., on delivering any riders) or \textit{available} (i.e., free to be assigned to a rider).
\end{definition}

When a new driver $d_j$ \revision{joins} the platform, he/she is considered to be available to serve riders. Once a rider $r_i$ is assigned to a driver $d_j$, the driver will move to the source location $s_i$ to pickup rider $r_i$, then send rider $r_i$ to his/her destination location $e_i$. During that period, driver $d_j$ is considered \textit{busy}. After driver $d_j$ finishes his/her current task,  $d_j$ will become available again.  
For region $a_k$ at timestamp $t$, we denote the set of available drivers as $D_k(t)$ and the number of them as $|D_k(t)|$.

\subsection{The Maximum Revenue Vehicle Dispatching Problem}

Before presenting the formal definition of the maximum revenue vehicle dispatching problem, we first define the valid rider-and-driver dispatching pair.

\begin{definition}[Valid Rider-and-Driver Dispatching Pair]
	Let $\langle r_i, d_j\rangle$ be a valid rider-and-driver dispatching pair, where driver $d_j$ can arrive at the pickup location $s_i$ of rider $r_i$ before the pickup deadline $\tau_i$ and  driver $d_j$ is in available status when he/she is picking up rider $r_i$.
\end{definition}

Now we give the formal definition of the maximum revenue vehicle dispatching problem as follows:
\begin{definition}[Maximum Revenue Vehicle Dispatching Problem, MRVD]
For a given time period $\mathbb{T}$, a set of impatient riders $R_{\mathbb{T}}$ and a set of drivers $D_{\mathbb{T}}$, the maximum revenue vehicle dispatching problem is to select a set, $I_{\mathbb{T}}$, of \revision{\textit{valid rider-and-driver dispatching pairs}} such that the overall revenue of the platform is maximized, which is:
\begin{equation}\label{eq:objective}
\max \sum_{\langle r_i, d_j \rangle \in I_{\mathbb{T}}} \alpha \cdot cost(s_i, e_i),
\end{equation}
\noindent where $\alpha$ is the travel fee rate of the platform.
\end{definition}

Intuitively, to maximize the overall revenue the platform  should serve as many long travel distance riders as possible as shown in Equation \ref{eq:objective}. However, the platform has no control on riders (i.e., the platform cannot schedule the arrivals and enlarge the waiting deadlines of riders), and can only affect the behaviors of drivers (i.e., dispatching drivers to pickup different riders). We will have a reduction in the end of this section to show practical rules to dispatching drivers to maximize the overall revenue of the platform.

\subsection{Hardness of MRVD}

\revision{
We prove MRVD is NP-hard through a reduction from a variant of  traveling salesman problem (TSP),  the deadline TSP \cite{bansal2004approximation}, which is a  known NP-hard problem. 

\begin{theorem}
	(Hardness of  MRVD) The problem of maximum revenue vehicle dispatching (MRVD) is NP-hard.
	\label{theorem:hardness}
\end{theorem}

\begin{proof}
	We prove the theorem by a reduction from the deadline TSP \cite{bansal2004approximation}.
	A deadline TSP problem can be described as follows:
	given a set of nodes $V$, each node $v_i \in V$ is located at location $v_i^l$ with a deadline $v_i^d$. There is a salesman $s$ locating at position $s^l$ at the beginning, who wants to visit the nodes. If $s$ can visit node $v_i$ before its deadline $v_i^d$, $s$ will receive a reward $v_i^r$. The problem is to find a path for $s$ to visit nodes such that the total reward is maximized.

	For the deadline TSP instance, we can transform it to an instance of
	MRVD as follows: we give only one driver $d_j$ with unlimited lifetime located at position $s^l$ at the very beginning. In addition, for each node $v_i \in V$, we generate a rider $r_i$, who is located at $v_i^l$ with a pickup deadline of $v_i^d$.  All the riders post their orders at the beginning of time. We set the travel fee rate $\alpha$ to a large enough value such that the destination $e_i$ of each rider $r_i$ is very close to his/her origin location $v_i^l$ and the travel time of serving $r_i$ can be ignored. In addition, the travel cost $\alpha\cdot cost(v_i^l,e_i)$ is equal to the visiting reward of the corresponding node $v_i^r$. Then, for this MRVD instance, we want to arrange a schedule for the given driver such that his/her overall revenue is maximized.
	
	Thus, to maximize the overall revenue satisfying the pickup deadlines of riders is same to maximize the total reward in the deadline TSP problem.
	
	Given this	mapping, it shows that the deadline TSP instance can
	be solved if and only if the corresponding MRVD problem instance can be solved.
	This way, we reduce the deadline TSP to the MRVD problem. Since the deadline TSP is known to be NP-hard \cite{bansal2004approximation}, MRVD is also NP-hard, which completes our proof.
\end{proof}

\begin{table}
	\begin{center}
		\caption{\small Symbols and Descriptions.} \label{tab:symbols}\vspace{-2ex}
			\begin{tabular}{l|l}
				{\bf Symbol} & {\bf \qquad \qquad \qquad\qquad\qquad Description} \\ \hline \hline
				$a_k$   & a region/grid\\
				$r_i$   & an impatient rider\\
				$o_i$   & the order of the impatient rider $r_i$\\
				$t_i$  & the timestamp when $r_i$ posts her ride request\\
				$s_i$   & the source location of rider $r_i$\\
				$e_i$   & the destination location of rider $r_i$\\
				$\tau_i$   & the pickup deadline of rider $r_i$\\
				$d_j$   & a driver\\
				$l_j(t)$   & the position of driver $d_j$ at timestamp $t$\\
				$D_k(t)$   & a set of available drivers in region $a_k$ at time $t$\\
				\hline
			\end{tabular}
	\end{center}
\end{table}

In real platform, orders are created at different timestamps, which means our MRVD problem is an online problem. To evaluate the effectiveness of algorithms on online problems, competitive ratio is a common used metric, which is the ratio of the result achieved by an online algorithm to the optimal result achieved in the corresponding offline problem. However, in the existing study about online deadline TSP problem \cite{wen2012deadline}, the authors prove that there is no algorithm can achieve a constant competitive ratio for online deadline TSP problem even the arriving timestamps of orders are known in advance. Moreover, in MRVD, usually multiple drivers need to be arranged, which means that the MRVD problem  is more complex than  the online deadline TSP problem. Thus, we turn to use the experimental results to show the effectiveness of our approaches.

}

\subsection{Reductions of MRVD}
\label{sec:reduction}
Let $T_j$ be the lifetime of driver $d_j$ from the time he/she joins  to the time he/she exits the platform. During $T_j$, the status of driver $d_j$ keeps switching between \textit{available} and \textit{busy}. We notice that only when driver $d_j$ is in \textit{busy} status, he/she contributes to the overall revenue of the platform. Then, we can rewrite the objective function of MRVD as below:
\vspace{-1ex}
\begin{align}
	&\max \sum_{\langle r_i, d_j \rangle \in I_{\mathbb{T}}} \alpha \cdot cost(s_i, e_i)\notag\\
	\Rightarrow & \max \sum_{d_j \in D_{\mathbb{T}}} \sum_{r_i \in R_j} \alpha \cdot cost(s_i, e_i)\notag\\
	\Rightarrow & \max \alpha \sum_{d_j \in D_{\mathbb{T}}} \sum_{r_i \in R_j}  cost(s_i, e_i) \label{eq:long_busy}
\end{align}
\noindent where $D_{\mathbb{T}}$ is the set of drivers on the platform during the given time period $\mathbb{T}$, and \revision{$R_j$ is the set of riders that are served by driver $d_j$ in the selected set, $I_{\mathbb{T}}$, of \textit{valid rider-and-driver dispatching pairs}.} According to Equation \ref{eq:long_busy}, the platform should maximize the length of the total busy time of each driver to maximize its overall revenue. Since the lifetime $T_j$ of each driver $d_j$ is fixed, to maximize his/her total \textit{busy} time, $\sum_{r_i \in R_j}cost(s_i, e_i)$, is equivalent to minimize his/her total idle time, $T_j - \sum_{r_i \in R_j}cost(s_i, e_i)$. Then, the objective of MRVD can be rewritten as follows:
\begin{align}
&\max \alpha \sum_{d_j \in D_{\mathbb{T}}} \sum_{r_i \in R_j}  cost(s_i, e_i)\notag\\
\Rightarrow &\min \sum_{d_j \in D_{\mathbb{T}}} \big(T_j - \sum_{r_i \in R_j}  cost(s_i, e_i)\big)\notag\\
\Rightarrow & \min \sum_{d_j \in D_{\mathbb{T}}}  \sum_{i =0}^{|R_j|}  \psi_{ij}\label{eq:idle}
\end{align}
\noindent where $\psi_{ij}$ is the idle time of driver $d_j$ after delivering rider $r_i$. Here, $\psi_{0j}$ indicates the idle time of driver $d_j$ before picking up his/her first rider.

According to Equation (\ref{eq:idle}), to maximize the overall revenue, the platform intuitively should reduce the number of served riders (e.g., $|R_j|$) and the idle time interval (e.g., $\psi_{ij}$) between any two consecutive riders for each driver. It may be confused that reducing the number of served riders for each driver seems to contradict the goal of maximizing the overall revenue. To explain this contradiction, we denote the time period of serving a rider and the idle time before serving the next rider as a \textit{service round} for a driver. In fact, the lifetime of a driver $d_j$ is fixed as $T_j$, when driver $d_j$ serves fewer riders, the average length of  service rounds will be longer (i.e., $\frac{T_j}{|R_j|}$). Then, for a service round that driver $d_j$ is assigned to serve rider $r_i$, minimizing the idle time interval $\psi_{ij}$ will lead to that the travel cost $cost(s_i,e_i)$ increases, which agrees with the intuition from Equation (\ref{eq:long_busy}). In conclusion, we can have two practical and controllable rules for the platform in its online processes to maximize the overall revenue during a given time period $\mathbb{T}$: \textbf{a) \textit{associating higher priorities to the riders whose travel costs are high}; b) \textit{reducing the length of the idle time between serving any two consecutive riders for each driver.}}

In the rest of this paper, we propose a queueing-theoretic framework, taking into consideration of the travel cost and the idle time length after each ride request, to maximize the overall revenue of the platform during a given time period.
\section{Overview of Queueing-Based Vehicle Dispatching Framework}
\label{sec:framework}

\begin{figure}[t!]\centering
	\scalebox{0.36}[0.36]{\includegraphics{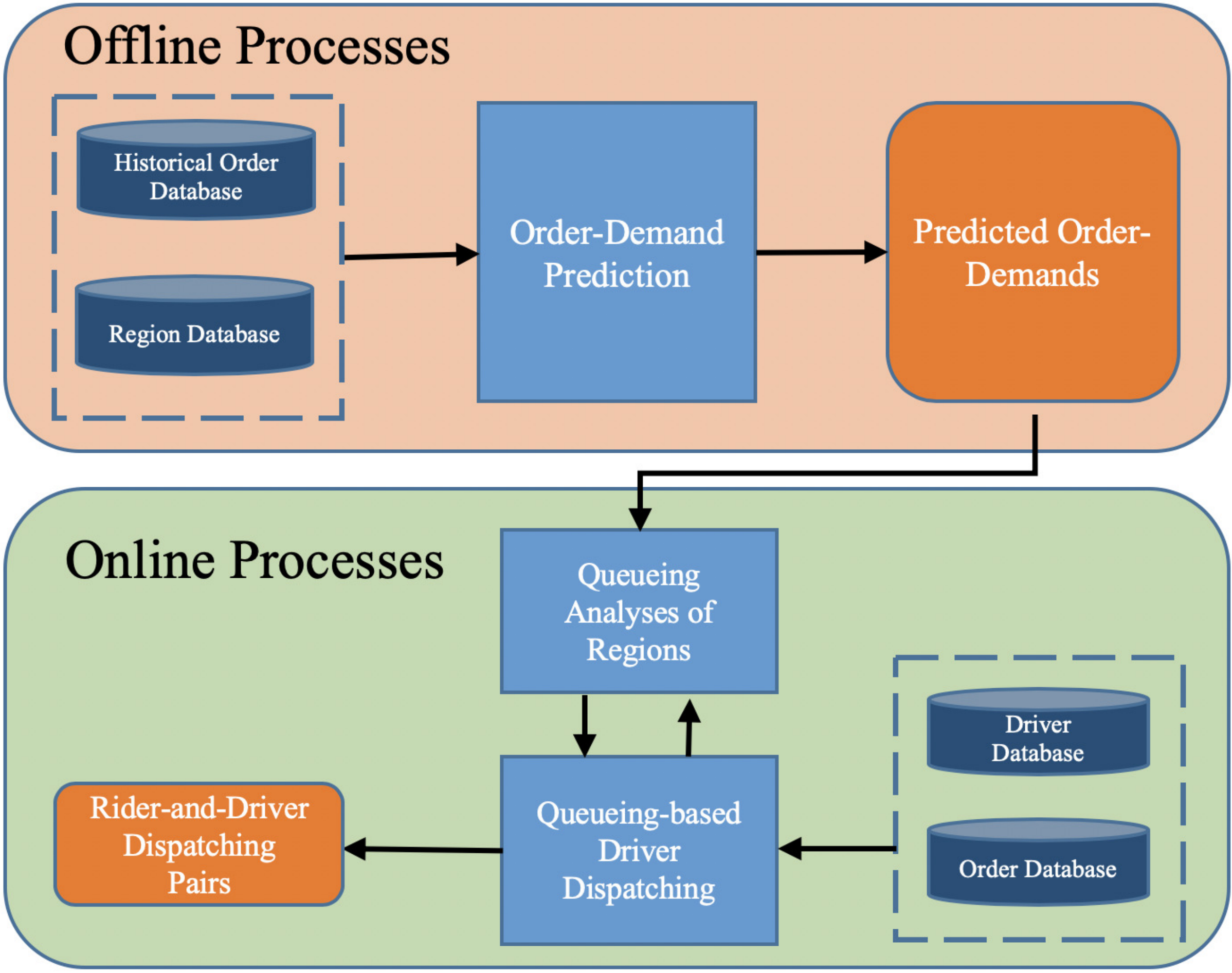}}
	\caption{\small Illustration of the Framework Work Flow.}
	\label{fig:framework}
\end{figure}

In this section, we introduce an overview of our queueing-based vehicle dispatching framework. In general, the framework includes three parts:  vehicle demand-supply prediction, queueing analysis of each region and queueing-based vehicle dispatching algorithms. From the other perspective, our framework contains offline prediction processes and online analysing and dispatching processes as shown in Figure \ref{fig:framework}. 
In the offline processes, the platform predicts the number of orders  for each region in each time periods based on the historical order records. In the online processes, the platform utilizes the predicted order demands and supplies to assign available drivers to orders with a goal to maximizing the overall revenue of the platform subjected to the deadline constraint of orders. In addition, the driver dispatching process and the queueing analyses process can affect each other, thus we interactively assign orders to drivers and update the results of queueing analyses of regions.

\revision{We first briefly introduce the major parts of our queueing-based vehicle dispatching framework, then propose a batch-based vehicle dispatching algorithm to handle the orders from riders.}

\subsection{Major Parts of Queueing-Based Vehicle Dispatching Framework}

\revision{Our queueing-based vehicle dispatching framework includes three major parts: offline vehicle demand-supply prediction, region queueing analysis and queueing-based vehicle dispatching algorithm.}
 
\subsubsection{Offline Vehicle Demand-Supply Prediction}
\label{sec:demand_prediction}
In our framework, we predict the order demand of each region for given time periods. For the rejoined drivers, we can estimate their availability based on their assignments and travel costs. In practice, it is hard to predict the accurate location and timestamp of a particular rider since the uncertain behaviors of a single user. To utilize the distribution of riders, we predict the number of riders for a given region (i.e., a spatial range of area, such as square regions or hexagon regions) in a given time period (i.e., next 30 minutes).  Existing work can be applied offline to predict the demand of riders  in given regions and time periods, such as demand-supply prediction of traffic \cite{li2015traffic, chu2018passenger}, and spatial-temporal data prediction \cite{zhang2016dnn, cressie2015statistics}. In this paper, we test the representative prediction algorithms, e.g., Historical Average (HA) method,  Linear Regression (LR) method, Gradient Based Regression Tree (GBRT) method \cite{friedman2002stochastic} and DeepST  \cite{zhang2017deep} on the real-world taxi demand-supply dataset and select the most effective one, DeepST \cite{zhang2017deep}, for our offline demand-supply prediction process,  which can achieve very accurate order demand prediction results (i.e., 2.3 \% RMSE) on our testing data set. Specifically, DeepST uses Convolutional Neural Network (CNN) \cite{krizhevsky2012imagenet} on historical data of order counts and meta data (e.g., time of day, day of week and city weather) to predict the order demand for each region in each time slot (e.g., a period of 30 minutes). Due to space limitation, we put the detailed comparison of the spatial temporal prediction models  in Appendix A of our technical report~\cite{report}.

\subsubsection{Region Queueing Analysis}
The available drivers in a region $a_x$ in a time period $\mathbb{T}$ \revision{come} from the rejoined active drivers and unassigned drivers in the  previous batch.
With the predicted numbers of orders for the region $a_x$ in a given time period $\mathbb{T}$ and the schedules of active drivers, we can know the demand and supply of drivers for the region $a_x$ in  time period $\mathbb{T}$. 
We estimate the waiting times  (idle time intervals) for vehicles from finishing last order to receive next order in Section \ref{sec:queue_analysis}. According to the analyses in Section \ref{sec:problem_definition}, shorter idle time intervals are better. Thus, the estimated waiting times of vehicles can be used to guide the order dispatching process to achieve a high overall revenue.

\subsubsection{Queueing-Based Vehicle Dispatching}
The platform needs to dispatch drivers to serve most ``valuable'' riders with high priorities. According \revision{to} the  analyses in Section \ref{sec:problem_definition}, orders having high travel costs and ending in ``hot regions'' (i.e., regions with many future orders) can contribute more to the platform, which should be associated with high priorities. In addition, since drivers usually prefer to serve riders close to their locations after they finish the last orders, the platform's selection on serving orders will affect the vehicle supply in the future, which in turn will affect  the queueing analyses of the related regions. In Example \ref{example1}, if the platform dispatch a driver to serve a rider having a destination in region $A_1$, the driver supply in region $A_1$ will increase slightly after finishing the order. We propose efficient and effective algorithms in Section \ref{sec:task_assignment} to dispatch available drivers to riders with an optimization goal of maximizing the overall revenue of the platform subjected to the deadline constraint of orders.

 \subsection{The Batch-based Vehicle Dispatching Algorithm}
 \label{sec:batch_framework}
 \begin{algorithm}[t]
 	\DontPrintSemicolon
 	\KwIn{\small The overall time period $\mathbb{T}$}
 	\KwOut{\small A set of rider-and-driver dispatching pairs within the time period $\mathbb{T}$}
 	
 	\While{current time $\bar{t}$ is in $\mathbb{T}$}{
 		\ForEach{$a_k \in A$}{
 			retrieve  the waiting riders in region $a_k$ to $R_k$\;
 			retrieve the available drivers in region $a_k$ to $D_k$\;
 			predict the number of upcoming riders in region $a_k$ during $[\bar{t}, \bar{t}+t_c]$ as $|\hat{R}_k|$\;
 			count the number of upcoming rejoined drivers in region $a_k$ during $[\bar{t}, \bar{t}+t_c]$ as $|\hat{D}_k|$\;
 		}

 		use  \textit{task-priority greedy} or \textit{local search} approach to obtain a  set of rider-and-driver pairs $I_{\bar{t}}$
 		
 		\ForEach{$\langle r_i, d_j \rangle \in I_{\bar{t}}$}{
 			inform driver $d_j$ to pick rider $r_i$\;
 		}
 		wait till $\bar{t}+\Delta$\;
 	}
 	
 	\caption{Batch-based Vehicle Dispatching Algorithm}
 	\label{alg:batch_framework}
 \end{algorithm}
 
 To handle the online processes of vehicle dispatching, we propose a batch-based processing framework to iteratively assign drivers to riders every $\Delta$ seconds. Note that, in real applications (e.g., DiDi Chuxing \cite{didi}), $\Delta$ is set very small (e.g., several seconds) such that the customers cannot notice the delay of the batch processing. To solve the assignment problem in each batch, we propose two heuristic algorithms to greedily maximize the revenue summation of the platform for the current scheduling time period $[\bar{t}, \bar{t}+t_c]$, where $\bar{t}$ indicates the current timestamp and $t_c$ is the length of the current scheduling time period. 

As shown in Algorithm \ref{alg:batch_framework}, we iteratively assign drivers to riders for multiple batches with a time interval $\Delta$  between every two successive batches. Specifically, for a batch starting at timestamp $\bar{t}$, we first retrieve
 a set, $R_k$, of waiting riders and a set, $D_k$, of available drivers for each region $a_k$ (lines 3-4). Here, waiting riders $R_k$ include the riders that are not assigned with any drivers during the last batch and the newly coming riders after the last batch in region $a_k$. Moreover, available drivers $D_k$ includes the drivers that are not assigned with any riders in the last batch, and the drivers that have finished the previous assigned tasks then rejoin the platform in region $a_k$. 
 To estimate the arrival rates of riders and serving rates drivers for the current scheduling time period $[\bar{t}, \bar{t}+t_c]$, we predict the number, $|\hat{R}_k|$, of upcoming riders and estimate the number, $|\hat{D}_k|$, of rejoin drivers in region $a_k$ (lines 5-6). Then, we use our proposed heuristic vehicle dispatching algorithms to achieve a set $I_{\bar{t}}$ of rider-and-driver dispatching pairs to greedily maximize the revenue summation of the platform for the current scheduling time period $[\bar{t}, \bar{t}+t_c]$ (line 7).  For every rider-and-driver dispatching pair $\langle r_i, d_j \rangle$ in $I_{\bar{t}}$, we inform the driver $d_j$ to pick up rider $r_i$ (lines 8-9). Finally, we wait until the time comes to the next batch $\bar{t}+\Delta$ (line 10).
 
 In the following sections, we will first introduce the queueing analyses of regions in Section \ref{sec:queue_analysis}, then propose our queueing based vehicle dispatching algorithms in Section \ref{sec:task_assignment}.
\section{Queueing Analyses of Regions}
\label{sec:queue_analysis}

In this section, we analyze the waiting riders for each single region through a queueing model. In queueing theory, customers join queue in an arrival (or ``birth'') rate $\lambda$, then the platform will serve the customers in a service (``death'') rate $\mu$. We first introduce the queue configuration, then estimate the idle time for a driver after he/she finishes his/her current order. 

\subsection{Queue Configuration of a Single Region}

In this paper, the platform can be considered as a server to match available drivers and waiting riders in each region. The riders come to the platform and wait for drivers to pick them. However, the riders are impatient and will leave the platform if they are not served before their deadlines. The available drivers come from the rejoined active drivers, who are the ones continuing to work on the platform after finishing their assigned orders. 

Similar to the previous assumption in the related work \cite{banerjee2016dynamic}, we assume the arrival rate of riders (in number per minute) follows the  Poisson distribution with rates $\lambda$ in a region $a$ during a short time period with length $t_c$ (e.g., a half hour). In addition, we also model the arrivals of rejoined active drivers follow a Poisson distribution with a  rate of $\mu$ in a region $a$ during a short time period with length $t_c$.  Note that, although the arrival rate of riders and drivers may change during different time periods in a day (e.g., 8 to 9 A.M. and 8 to 9 P.M.), to facilitate the analysis of the queueing situation in a short time period (e.g., a half hour), we model the arrival rates of riders and drivers as stable rates. We verify our assumption that the arrivals of orders and rejoined drivers follow Poisson distributions through chi-square ($\chi^2$) tests \cite{greenwood1996guide}. Due to the space limitation, please refer to Appendix B of our technical report \cite{report} for  details.

\revision{When the riders are more than the drivers in a region, the platform will select a subset of riders with ``higher priorities'' to serve first. Usually a driver will rejoin the platform in the same region of the destination of her/his last served rider. It leads to that the drivers appear in the regions where the destinations of the selected high-priority riders are, and then the arrival rates of drivers in the corresponding ``selected'' regions will increase.}
In our queueing model, the priority of a rider is determined in line with his/her travel cost and the demand-supply situation in his/her destination region. According to the analysis in Section \ref{sec:reduction}, to improve the overall revenue, the platform prefers to give higher priorities to the riders who have higher travel costs and are going to ``hot'' regions.

Figure \ref{fig:birth_death} illustrates the birth-death chain of the queueing model for a region, where each circle indicates a state and the numbers in the circles represent the numbers of waiting riders in the corresponding states. For example, the state of $\textbf{2}$ indicates that there are 2 waiting riders in the region. Each link represents the transfer event from the tail state to the head state along its direction, where the value close to the link indicates the transfer rate. For example, the link with value $\lambda_0$ pointing from state \textbf{0} to state \textbf{1} indicates the transfer rate  from  state \textbf{0} to state \textbf{1} is $\lambda_0$. Since drivers may also congest in a region, if the arrival rate of drivers is higher than that of riders (i.e., $\mu > \lambda$), to uniformly represent the queueing situation of a region, we utilize the state of $\textbf{-n}$ ($-n<0$) to indicate that there are $n$ congested drivers  in the region.

\begin{figure}[t!]\centering\vspace{-2ex}
	\scalebox{0.4}[0.4]{\includegraphics{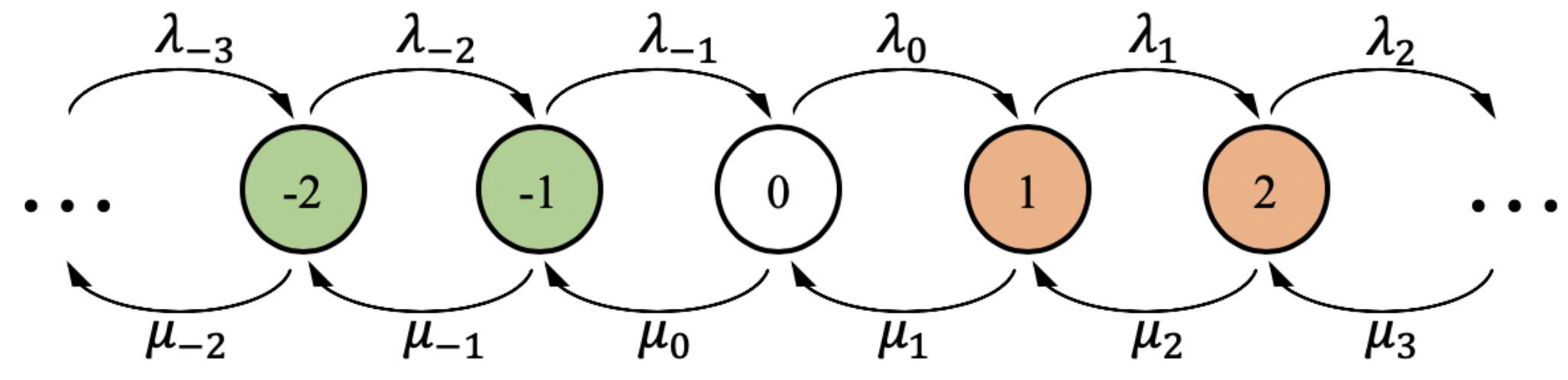}}\vspace{-2ex}
	\caption{\small Birth-death chain for the queue of a region.}
	\label{fig:birth_death}
\end{figure}


Another issue to change the number of waiting riders is that the impatient riders may quit from the platform if they are not served after a time period, which is called \textit{reneging} in queueing theory. As defined in the existing work \cite{shortle2018fundamentals}, we can define a state related reneging function $\pi(n)$ as the reneging rate of riders for the state $\textbf{n}$ ($n>0$) of the birth-death chain in Figure \ref{fig:birth_death}. As suggested in \cite{shortle2018fundamentals}, a good practice for the reneging function $\pi(n)$ is to define it as $e^{\beta n/\mu}$, where $\beta$ is a parameter determined based on the historical reneging records in the corresponding region.
Then, the death/service rate $\mu_n$ of the state $\textbf{n}$ can be adjusted as follows:

\begin{equation}
\mu_n=\left\{
\begin{array}{ll}
\mu, & n \leq 0 \\
\mu + \pi(n), & n > 0
\end{array}
\right. \label{eq:service_rate}
\end{equation}

For the birth/arrival rate $\lambda_n$ of state $\textbf{n}$, we define it as $\lambda_n=\lambda$, since drivers do not renege in our queueing model.

\subsection{Expected Idle Time Interval of Drivers}

\revision{In this section, we analyze the expected idle time interval of a driver $d_j$. Let region $a$ be the destination region of the last rider $r_i$ of $d_j$ and $d_j$ will join the queue of region $a$ after serving $r_i$. Thus, the state of the region $a$ (i.e., the length of waiting drivers or riders) will directly affect the waiting time (idle time) of $d_j$ before serving the next rider. For example, if the region $a$ is in a state of $\textbf{n}_a$ and $n_a>0$, which means there are $n_a$ riders are waiting for drivers, driver $d_j$ can be immediately assigned with a new rider after finishing the last order. On the contrary, if the region $a$ is in a state of $\textbf{n}_b$ and $n_b < 0$, which means there are $n_b$ available drivers are waiting for riders, driver $d_j$ will not be assigned with any new riders before $n_b$ available drivers are assigned with riders first. For  region $a$, we assume the arrives of the riders and rejoined drivers follow Poisson distributions with rates $\lambda$ and $\mu$, respectively. Then, the region $a$ can be in any state $\textbf{n}$ with the corresponding probability $p_n$.}

\begin{figure}[h!]\centering
	\scalebox{0.43}[0.43]{\includegraphics{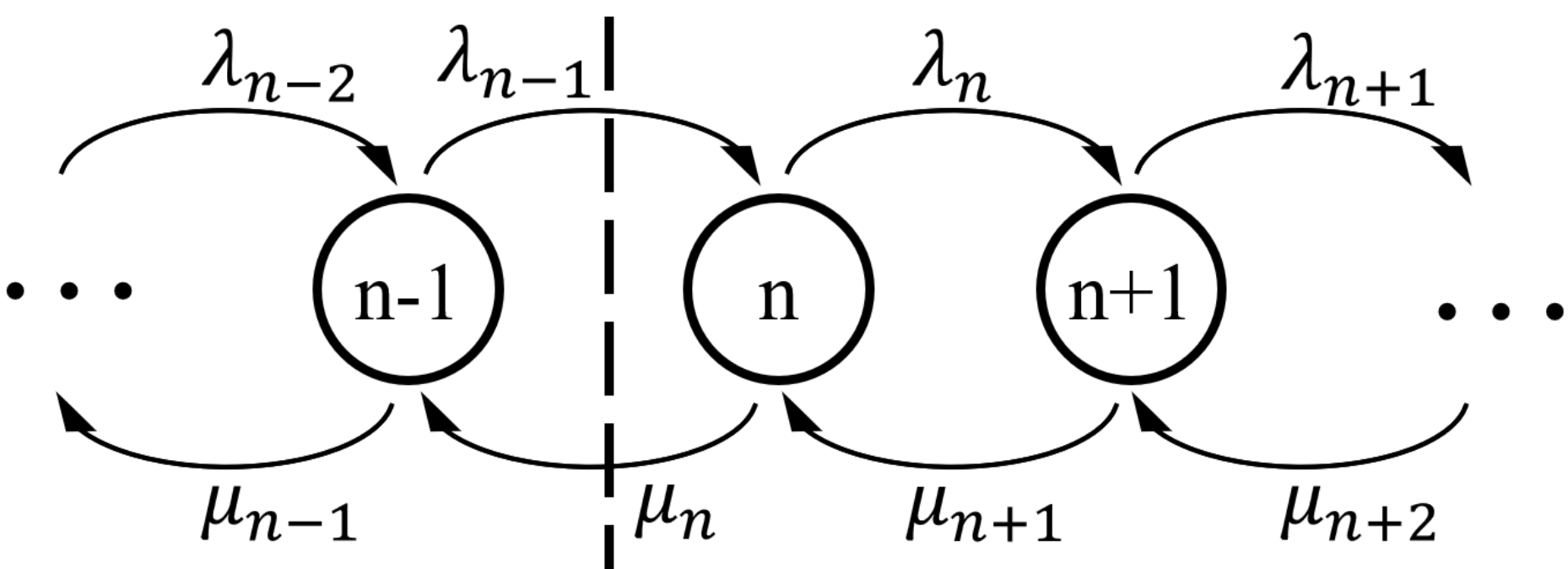}}
	\caption{\small Flow balance between states.}
	\label{fig:flow_balance}
\end{figure}

We here briefly introduce the flow balancing rule in analyzing the recursive relation between $p_n$ and $p_{n-1}$  \cite{shortle2018fundamentals}. As shown in Figure \ref{fig:flow_balance}, the mean flows across the dash line must be equal in a steady situation for the queue. In other words, for a relatively long period (e.g., 30 minutes), the rate of transitions $(\lambda_{n-1}p_{n-1})$ from state $\textbf{n-1}$ to state $\textbf{n}$ must equal the rate of transitions $(\mu_n p_n)$ from state $\textbf{n}$ to state $\textbf{n-1}$, which is as follows:
\begin{equation}\label{eq:flow_banlance}
\mu_n p_n = \lambda_{n-1}p_{n-1}
\end{equation}

By iteratively applying Equation (\ref{eq:flow_banlance}), we can derive:

\begin{equation}
p_n=\left\{
\begin{array}{ll}
p_0 \cdot(\frac{\mu}{\lambda})^{-n} , & n < 0 \\
p_0 \cdot \prod_{i=1}^{n}\frac{\lambda}{\mu + \pi(i)}, & n > 0
\end{array}
\right. \label{eq:p_n}
\end{equation}

\revision{Different from the traditional queueing model with only one-sided queue, our queueing model of a single region is  a double-sided queue. We need to analyze the particular idle times of drivers with our new queueing model.} We  estimate the expected idle time in different conditions: 1) more riders arrive; 2) more drivers rejoin; 3) balanced riders and drivers, since they have different properties.

\subsubsection{More Riders Arrive ($\lambda > \mu$)}
When $\lambda > \mu$, the summation of probabilities of states $\{\textbf{n}\}$, $n < 0$, can be calculated through a summation of geometric sequence, which is:
\begin{equation}\label{eq:left_sum}
\sum_{i=-1}^{-\infty} p_i =  p_0 \sum_{i=1}^{\infty} \big(\frac{\mu}{\lambda}\big)^{i}= \frac{\mu}{\lambda - \mu} p_0
\end{equation}
 
In addition, we have the fact that the summation of the probabilities $\{p_n\}, n\in[-\infty, \infty]$, must be 1.
\begin{equation}\label{eq:state_summation1}
\sum_{i=-1}^{-\infty} p_i  + p_0 + \sum_{i=1}^{\infty} p_i  = 1
\end{equation}

Putting Equations  (\ref{eq:p_n}) and (\ref{eq:left_sum})  into Equation \ref{eq:state_summation1}, we have:
\begin{equation}
p_0 \Big(\frac{\mu}{\lambda - \mu} + 1 + \sum_{n=1}^{\infty} \prod_{i=1}^{n} \frac{\lambda}{\mu + \pi(i)}\Big)=1\notag
\end{equation}

Then, we have 
\begin{equation}\label{eq:p_0}
p_0 = \Big(\frac{\lambda}{\lambda - \mu} + \sum_{n=1}^{\infty} \prod_{i=1}^{n} \frac{\lambda}{\mu + \pi(i)}\Big)^{-1}
\end{equation}

Drivers are dispatched in a first come, first served order. Let $T(n)$ be the expected idle time interval of an arrival driver $d_j$ when the queue is in state $\textbf{n}$. When there are waiting riders ($n>0$), the idle time interval of $d_j$ is just the processing time of the platform to arrange a new rider, which can be ignored. When there is no waiting riders ($n\leq 0$), the driver needs to wait until the next ($|n|  +1$)th rider appears in the region, which  needs $\frac{|n| + 1}{\lambda}$ time on average. Finally, we can estimate the expected idle time, $ET(\lambda, \mu)$, of a driver after join a queue of region having the arrival rate $\lambda$ of riders and the arrival rate $\mu$ of drivers as follows:

\begin{eqnarray}
ET(\lambda, \mu) &=&  \sum_{i=0}^{-\infty} \frac{|n| + 1}{\lambda}\cdot p_{i} + \sum_{i=1}^{\infty} 0 \cdot p_i \notag\\
&=& \frac{p_0}{\lambda} \sum_{i=0}^{\infty} (i + 1) \big(\frac{\mu}{\lambda}\big)^i \notag\\
&=& \frac{\lambda p_0}{(\lambda - \mu)^2} \label{eq:et}
\end{eqnarray}

\subsubsection{More Drivers Rejoin ($\lambda < \mu$)}

We notice that when $\lambda < \mu$, the queue will congest more and more drivers when time elapses, which will harm the efficiency of the platform much. The platform will avoid that the rate of drivers $\mu$ become larger than the rate of riders $\lambda$ for each region. However, when there are indeed more drivers rejoining, we still can estimate the expected idle time. 

Let $K$ be the number of available drivers during the current scheduling time period with length $t_c$. Then, the queue of the region can at most congest with $K$  drivers. Let $\theta = \frac{\mu}{\lambda}$. We can calculate the summation of probabilities of states  $\{\textbf{n}\}$, $-K \leq n < 0$, with the equation as follows:

\begin{equation}\label{eq:limited_left_sum}
\sum_{i=-1}^{-K} p_i =  p_0 \sum_{i=1}^{K} \big(\frac{\mu}{\lambda}\big)^{i}= \frac{\theta^{K+1} - \theta}{\theta - 1} p_0
\end{equation}

Then, we can update Equation \ref{eq:p_0} when $\lambda < \mu$ as follows:

\begin{equation}\label{eq:p_0_driver}
p_0 = \Big(\frac{\theta^{K+1} - 1}{\theta - 1} + \sum_{n=1}^{\infty} \prod_{i=1}^{n} \frac{\lambda}{\mu + \pi(i)}\Big)^{-1}
\end{equation}

In addition, when the expected number of rejoined drivers during the current scheduling time period with length $t_c$ is $K$ and $\lambda < \mu$, we can estimate the expected idle time $ET(\lambda, \mu)$ as follows:
\begin{equation}
ET(\lambda, \mu) = \frac{p_0}{\lambda}  \frac{ (K+1)\theta^{K+2} - (K+2)\theta^{K+1}  + 1  }{(\theta -1)^2}. \label{eq:et_drive}
\end{equation}


\subsubsection{Balanced Riders and Drivers ($\lambda = \mu$)} When $\lambda = \mu$, we can update Equation \ref{eq:limited_left_sum} as follows:

\begin{equation}\label{eq:limited_left_equal_sum}
\sum_{i=-1}^{-K} p_i =  p_0 \sum_{i=1}^{K} \big(\frac{\mu}{\lambda}\big)^{i}=K p_0
\end{equation}

Then, we have

\begin{equation}\label{eq:p_0_equal}
p_0 = \Big(K+1 + \sum_{n=1}^{\infty} \prod_{i=1}^{n} \frac{\lambda}{\mu + \pi(i)}\Big)^{-1}
\end{equation}

Next, we can estimate the expected idle time $ET(\lambda, \mu)$ when $\lambda=\mu$ as follows:

\begin{equation}
ET(\lambda, \mu) = p_0 \frac{(K+1)(K+2)}{2\lambda}\label{eq:et_equal}
\end{equation}
\section{Queueing-based Vehicle Dispatching Algorithms}
\label{sec:task_assignment}

\subsection{The Idle Ratio Oriented Greedy Approach}

\begin{algorithm}[t]
	\DontPrintSemicolon
	\KwIn{\small A set of Regions $A$, current timestamp $\bar{t}$}
	\KwOut{\small A set of rider-and-driver dispatching pairs $I_{\bar{t}}$}
	$I_{\bar{t}} \gets \{\emptyset\}$\;
	$I_{v} \gets \{\emptyset\}$\;
	\ForEach{$a_k \in A$}{
		retrieve a set $I_k$ of valid rider-and-driver dispatching pairs from $R_k$ and $D_k$\;
		$I_{v} \gets I_{v} \cup I_k$\;
		estimate the arrival rate $\lambda_{(k)}$ of riders and arrival rate $\mu_{(k)}$ of rejoined drivers in region $a_k$ during $[\bar{t}, \bar{t}+t_c]$\;
	}
	sort dispatching pairs in $I_v$ based on their idle ratio\;
	\While{$I_v$ is not empty}{
		select the rider-and-driver pair $\langle r_i, d_j\rangle$ having the smallest idle ratio from $I_v$\;
		add $\langle r_i, d_j\rangle$ to $I_{\bar{t}}$\;
		update $\mu_{(k)}$ of the destination region $a_k$ of $r_i$\;
		remove $\langle r_i, .\rangle$ and $\langle., d_j\rangle$ from $I_v$\;
	}
	
	\Return $I_{\bar{t}}$\;
	\caption{Idle Ratio Oriented Greedy Algorithm}
	\label{alg:idle_ratio_greedy}
\end{algorithm}

We first propose an \textit{idle ratio oriented greedy} approach to solve each batch process in line 7 of Algorithm \ref{alg:batch_framework} with a goal to maximize the revenue summation of the platform during the current scheduling time period $[\bar{t}, \bar{t}+t_c]$, where $\bar{t}$ is the current timestamp and $t_c$ is the length of the current scheduling time window. We first define the idle ratio of driver $d_j$ to server rider $r_i$, whose destination $e_i$ is in region $a_k$, as follows:

\begin{equation}\label{eq:idle_ratio}
IR(r_i, d_j) = \frac{ET(\lambda_{(k)}, \mu_{(k)})}{cost(s_i, e_i) + ET(\lambda_{(k)}, \mu_{(k)})},
\end{equation}

\noindent where $ET(\lambda_{(k)}, \mu_{(k)})$ is the expected idle time of driver $d_j$ when he/she rejoins the platform at region $a_k$, and $cost(s_i, e_i)$ is the travel cost (travel time) on serving rider $r_i$. Recall that, in Section \ref{sec:reduction}, we have two guiding rules for the platform to maximize its overall revenue after analyzing the MRVD problem: \textbf{a) \textit{associating higher priorities to the riders whose travel costs are higher}; b) \textit{reducing the length of the idle time between serving any two consecutive riders for each driver.}} We notice that when the travel cost $cost(s_i, e_i)$ increases, $IR(r_i, d_j)$ will decrease; when the expected idle time $ET(\lambda_{(k)}, \mu_{(k)})$ decreases, $IR(r_i, d_j)$ will also decrease. As a result, we only need to greedily select the rider-and-driver dispatching pairs with low idle ratios (as defined in Equation \ref{eq:idle_ratio}), then we can follow the above mentioned two guiding rules to maximize the overall revenue of the platform. Based  on the observation, we propose an idle ratio oriented greedy approach as shown in Algorithm \ref{alg:idle_ratio_greedy}, which greedily \revision{selects} the rider-and-driver dispatching pair having the smallest idle ratio value in each iteration.

Specifically, we first initialize the selected rider-and-driver pairs $I_{\bar{t}}$ and the valid rider-and-driver pairs $I_v$ with empty sets (lines 1-2). Then, for each region $a_k$, we put the valid rider-and-driver pairs $I_k$ between the waiting riders $R_k$ and available drivers $D_k$ in the region into $I_v$ (lines 4-5) and estimate the arrival rates, $\lambda_{(k)}$ and $\mu_{(k)}$, of new riders and rejoined drivers during the current scheduling period $[\bar{t}, \bar{t}+t_c]$ as follows:

\begin{equation}
\lambda_{(k)}=\left\{
\begin{array}{ll}
\frac{|\hat{R}_k|}{t_c}, & |R_k| \leq |D_k|\\
\frac{|\hat{R}_k| + |R_k| - |D_k| }{t_c}, & |R_k| > |D_k|
\end{array}
\right. \label{eq:region_lambda}
\end{equation}

\begin{equation}
\mu_{(k)}=\left\{
\begin{array}{ll}
\frac{|\hat{D}_k| + |D_k| - |R_k|}{t_c}, & |R_k| \leq |D_k|\\
\frac{|\hat{D}_k|}{t_c}, & |R_k| > |D_k|
\end{array}
\right. \label{eq:region_mu}
\end{equation}

\noindent where $|\hat{R}_k|$ and $|\hat{D}_k|$ are the numbers of predicted riders and future rejoined drivers in region $a_k$ during $[\bar{t}, \bar{t} + t_c]$. Next, after retrieving all the valid pairs, we sort them based on their idle ratios calculated with Equation \ref{eq:idle_ratio} (line 7). Note that, the expected idle time $ET(\lambda_{(k)},\mu_{(k)} )$ is determined by the arrival rates, $\lambda_{(k)}$ and $\mu_{(k)}$,  of new riders and rejoined drivers in the destination region $a_k$, thus we only need to estimate that for each region but not for each rider-and-driver pair individually. In each iteration of the while-loop (lines 8-12), we select the rider-and-driver pair $\langle r_i, d_j\rangle$ having the smallest idle ratio and remove its related pairs, $\langle r_i, .\rangle$ and $\langle., d_j\rangle$, of rider $r_i$ and driver $d_j$ from $I_v$ (since each driver only can serve one rider at one time). The selected pair $\langle r_i, d_j\rangle$ is added in $I_{\bar{t}}$ (line 10) and all the selected pairs $I_{\bar{t}}$ will be finally returned (line 13).

\noindent \textbf{Complexity Analysis.} Let the number of total waiting riders be $m$, the number of total available drivers be $n$ and the number of total regions be $x$. Assume riders and drivers be evenly distributed in $x$ regions and $x$ is much smaller than $m$ and $n$. In lines 3-6 of algorithm \ref{alg:idle_ratio_greedy}, retrieving all the valid rider-and-driver pairs needs $O(\frac{mn}{x})$. To sort the valid pairs in $I_v$ needs $O(\frac{mn}{x}\log_2(\frac{mn}{x}))$ (line 7). In each iteration of the while-loop (lines 8 - 12 of Algorithm \ref{alg:idle_ratio_greedy}), selecting the pair having the smallest idle ratio from sorted $I_v$ needs $O(1)$ (lines 9-10); updating $\mu_{(k)}$ and the idle ratio of average $\frac{mn}{x^2}$ related pairs  needs $O(\frac{mn}{x^2})$ (line 11); removing the related valid pairs $\langle r_i, .\rangle$ and $\langle., d_j\rangle$ from $I_v$ needs $O(\max(\frac{n}{x}, \frac{m}{x}))$ (line 12). Since in each iteration, at least one rider and one driver will be matched, thus there will be at most $\min(m,n)$ iterations. Then the complexity of the while-loop is $O(\frac{\min(m,n)mn}{x^2})$. Thus, the complexity of Algorithm \ref{alg:idle_ratio_greedy} is $O(\max(\frac{mn}{x}\log_2(\frac{mn}{x}), \frac{\min(m,n)mn}{x^2}))$. If we consider $x$ as a constant number, and $m$ is linearly related to $n$, the complexity can be considered as $O(n^3)$.

\subsection{The Local Search Algorithm}

In the idle ratio oriented greedy approach, we greedily select the pair having the ``current'' smallest idle ratio. However, the arrival rate $\mu_k$ of rejoined drivers in region $a_k$ will change after selecting some riders whose destinations are in $a_k$. \revision{Thus, the idle ratios of early selected rider-and-driver pairs may slightly increase in later iterations.} To overcome this shortcoming in the idle ratio oriented greedy approach,  we will propose a local search algorithm to improve the results, which keeps searching for rider-and-driver pairs $\langle r'_i, d_j\rangle$ with a smaller idle ratio for driver $d_j$ and update the assigned rider of $d_j$ to $r'_i$ until no such pairs can be found.

\begin{algorithm}[t]
	\DontPrintSemicolon
	\KwIn{\small A set of Regions $A$, current timestamp $\bar{t}$}
	\KwOut{\small A set of rider-and-driver dispatching pairs $I_{\bar{t}}$}
	Obtain a set, $I_{\bar{t}}$, of pairs with Algorithm \ref{alg:idle_ratio_greedy}\;

	\Do{$FLAG$ is $True$}{
		$FLAG \gets False$\;
		\ForEach{$\langle r_i, d_j\rangle \in I_{\bar{t}}$}{
			\ForEach{$r'_i \in R_j$}{
				\If{$IR(r'_i, d_j) < IR(r_i, d_j)$}{
					update $\langle r_i, d_j\rangle$ to $\langle r'_i, d_j\rangle$\;
					$FLAG \gets True$\;
				}
			}
		}
	}

	\Return $I_{\bar{t}}$\;
	\caption{Local Search Algorithm}
	\label{alg:local_search}
\end{algorithm}

Specifically, in Algorithm \ref{alg:local_search}, we first obtain a set, $I_{\bar{t}}$, of rider-and-driver pairs for current timestamp $\bar{t}$ through Algorithm \ref{alg:idle_ratio_greedy} (note that, we can also obtain $I_{\bar{t}}$ through any other algorithms). Then, in each iteration, we check whether the rider of a pair $\langle r_i, d_j \rangle \in I_{\bar{t}}$ can be replaced by any other valid rider $r'_i\in R_j$ for $d_j$, where $R_j$ is the valid riders to $d_j$. If no replacement happens, we will return the updated set, $I_{\bar{t}}$, of the selected rider-and-driver pairs.

We prove our local search algorithm can converge. We prove it in the below lemma.

\begin{lemma}
The local search algorithm (Algorithm \ref{alg:local_search}) can converge.
\end{lemma}
\begin{proof}
Assume Algorithm \ref{alg:local_search} cannot converge. Then, there is at least one driver $d$ who keeps switching between two riders $r_u$ and $r_v$. When $d$ selects $r_u$, we have $IR(r_u, d)<IR(r_v, d)$; \revision{otherwise, we have $IR(r_u, d) \geq IR(r_v, d)$.} We denote the regions where $r_u$ and $r_v$ will end as $a_u$ and $a_v$, respectively.

\revision{Since the travel costs of $r_u$ and $r_v$ do not change, different $ET(\lambda_{u}, \mu_{u})$ and $ET(\lambda_{v}, \mu_{v})$ lead to different $IR(r_u,d)$ and $IR(r_v,d)$. Specifically, according to the definition of $IR(r,d)$ in Equation \ref{eq:idle_ratio}, $IR(r_u,d)$ is positively correlated with $ET(\lambda_{u}, \mu_{u})$ (e.g., when $ET(\lambda_u, \mu_u)$ increases, $IR(r_u,d)$ will also increase). When more drivers rejoin in region $a_u$, the expected waiting time $ET(\lambda_{u}, \mu_{u})$ will increase.

Let driver $d$ select rider $r_u$ at the beginning. If driver $d$ switches from $r_u$ to $r_v$ in some iteration $\zeta$, there must be some more rejoined drivers switch to region $a_u$, which leads to $IR(r_u, d)>IR(r_v, d)$. Thus, there must be at least one other driver $d'$ who  switches from his/her valid rider $r'_v$ to a new rider $r'_u$ whose destination is also in region $a_u$ (i.e., $IR(r'_u, d')<IR(r'_v, d')$).  After $d$ switches to $r_v$, the number of rejoined drivers in region $a_u$ will decrease, and $IR(r'_u, d')$ will also decrease. As a result, $d'$ will not switch back to $r'_v$. Since no drivers will switch out from region $a_u$, $IR(r_u, d)$ will at least not decrease. As a result, $d$ will not switch back to $r_u$, which is contradicted with the assumption that $d$ keeps switching between $r_u$ and $r_v$. Thus Algorithm \ref{alg:local_search} can converge.}
\end{proof}

%

\noindent \textbf{Complexity Analysis.} Let the number of total waiting riders be $m$, the number of total available drivers be $n$ and the number of total regions be $x$. Assume that riders and drivers are evenly distributed in $x$ regions and $x$ is much smaller than $m$ and $n$. The number of total valid rider-and-driver pairs will be $O(\frac{mn}{x})$. The number, $|R_j|$, of valid riders for driver $d_j$ will be $O(\frac{n}{x})$. Then each iteration of the while-loop needs $O(\frac{mn^2}{x^2})$. Let $L_{max}$ be the maximum iteration numbers, then the complexity of Algorithm \ref{alg:local_search} will be $O(L_{max}\frac{mn^2}{x^2})$. If we consider $x$ and $L_{max}$ as constant numbers, and $m$ is linearly related to $n$, the complexity can be considered as $O(n^3)$.
\section{Experimental Study}
\label{sec:experimental}

In this section, we show the efficiency and effectiveness of our queueing-theoretic framework with different vehicle dispatching algorithms embedded through experimental studies on both synthetic and real datasets. 

\subsection{Data Sets}
We use both real and synthetic data to test our framework. Specifically, for the real data set, we use the taxi trip data sets in NYC~\cite{nyctlc}. 

\noindent \textbf{New York Taxi Trip Dataset.}
New York Taxi and Limousine Commission (TLC) Taxi Trip Data \cite{nyctlc} is a dataset recording the information of taxi trips  in New York, USA. The records are collected and provided to the NYC Taxi and Limousine Commission technology under the Taxicab \& Livery Passenger Enhancement Programs (TPEP/LPEP \cite{TPEP}). Trip records can be categories as three types: yellow taxi, green taxi and FHV (For Hire Vehicle). However, due to the privacy issues, only the locations of yellow taxi can be access in the dataset long time ago. In addition, the number of FHV and green taxi records is much smaller than that of yellow taxi. Thus, we only use the taxi trip records of yellow taxis in our experiments. Each trip record includes its pick-up and drop-off taxi-zones, GPS locations and timestamps, the number of passengers and the total travel cost.  In our experiment, we use  taxi trip data records from  January 1st, 2013 to  May 20th,  2013 as training data set and  May  28th, 2013 as the test data set. \revision{In the taxi records of May  28th, 2013, there is 282,255 orders. Figure \ref{fig:order_sample} shows the pick-up locations of orders  from 8:00 A.M. to 8:45 A.M. in New York.}

\begin{figure}[t!]\centering\vspace{-2ex}
	\scalebox{0.35}[0.35]{\includegraphics{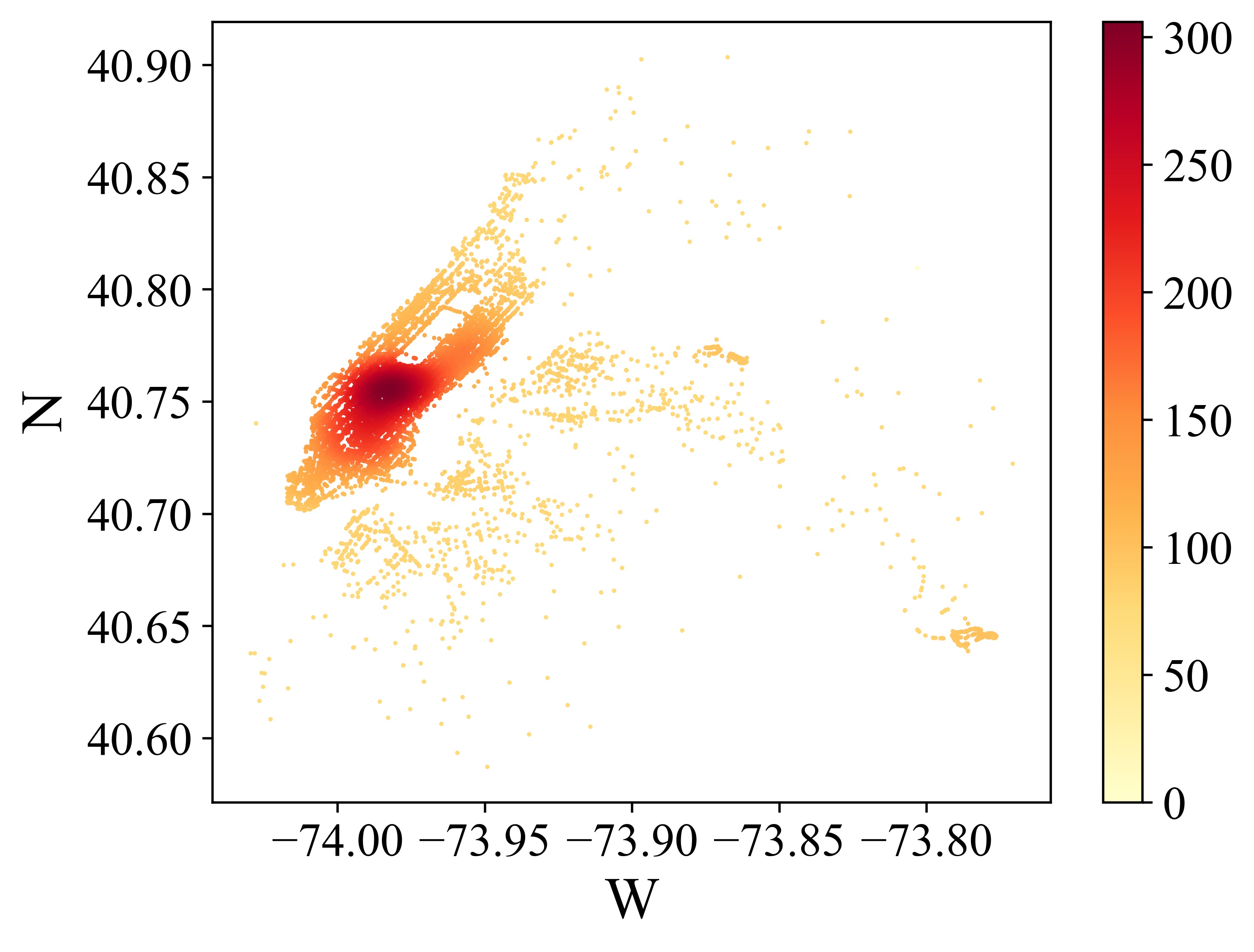}}\vspace{-3ex}
	\caption{\small  Distribution of Orders. }
	\label{fig:order_sample}
\end{figure}

\subsection{Experimental Configurations}

For the experiments on the real data set, we use the pickup location and timestamp of a taxi trip record to initialize the source location $s_i$ and the posting timestamp $t_i$ of a ride order $r_i$. Then the dropoff location of the taxi trip record is used to set the destination $e_i$ of the ride order. \revision{Thus, there are 282,255 riders in our experiments.} For the pickup deadline $\tau_i$ of rider $r_i$, we configure it by adding a uniform random noise $\tau' \in [1, 10]$ and a base pickup waiting time $\tau$ (configured with the setting in Table \ref{tab:settings}) to the posting timestamp $t_i$ (e.g., $\tau_i = t_i + \tau' + \tau$). \revision{To initialize the origin locations of drivers at the beginning timestamp $0$, we first randomly select a set of order records and use their pick-up locations as the origin locations of drivers. The number of drivers are configured as the parameter $n$ in Table \ref{tab:settings} from 1K to 5K.} The whole space of New York City area (i.e., \ang{-73.77}$\sim$\ang{-74.03}, \ang{40.58}$\sim$\ang{40.92}) is evenly divided into 16$\times$16 grids.

In our experiment, we run the batch process every time period $\Delta$. To estimate the arrival rate of new riders and rejoined drivers, we look up a time window of length $t_c$ with the ``current'' timestamp $\bar{t}$ as the beginning time of the time window.

\subsection{Approaches and Measurements}
We conduct experiments to evaluate the effectiveness and efficiency of our queueing-theoretic  vehicle dispatching framework with two batch processing vehicle dispatching algorithms, namely \textit{idle ratio oriented greedy} (IRG) and local search (LS), in terms of the total revenue and the average batch running time. Note that, we set the parameter $\alpha$ as 1, such that the total revenue is equal to the total serving time (e.g., the total travel cost of   served ride orders).

Specifically, for IRG (or LS) we can further have two different combinations: IRG-P and IRG-R (or LS-P and LS-R), which use the predicted taxi demand and the real taxi demand, respectively. In addition, we also compare our approaches with three baseline methods: (1) \textit{long trip greedy} (LTG), which greedily assigns orders with the highest revenue to available taxis; (2) \textit{nearest trip greedy} (NEAR), which greedily assigns the nearest order to each available taxi; (3) \textit{random} (RAND), which randomly assigns orders to available taxis. We also compare our methods with the state-of-the-art solution, POLAR \cite{tong2017flexible}, on car-hailing problem, which utilizes the predicted number of orders and drivers to conduct an offline bipartite matching first, then uses the offline result as a blueprint to guide the online task matching.  In addition, we report the upper bound (UPPER) by summing up the revenue of  the most expensive orders that can be served by idle drivers ignoring their pick-up distances in each batch. Our framework can also handle the target of maximizing the number of total served orders through modifying IRG to select the order with the smallest summation of its travel cost and expected idle time in each iteration. Due to space limitation, please refer to Appendix C of our technical report \cite{report} for more details of  maximizing the number of  total served orders.

\begin{table}[t]
	\begin{center}
		{\small
			\caption{\small Experimental Settings.} \label{tab:settings}\vspace{-2ex}
			\begin{tabular}{l|l}
				{\bf \qquad \qquad \quad Parameters} & {\bf \qquad \qquad \qquad Values} \\ \hline \hline
				the number, $n$, of drivers  & 1K, 2K, {3K}, \textbf{4K}, 5K\\
				base pickup waiting time, $\tau$ (seconds)  & 60, \textbf{120}, 180, 240, 300 \\
				the length of batch interval, $\Delta$ (seconds) & \textbf{3}, 5, 10, 20, 30\\
				the length of time window, $t_c$ (minutes)& \textbf{5}, 10, 15, 20, 40, 60, 80, 100\\
				\hline
			\end{tabular}
		}\vspace{-2ex}
	\end{center}
\end{table} 

Table \ref{tab:settings} shows the settings of our experiments, where the default values of parameters are in bold font. In each set of experiments, we vary one of the parameters and keep other parameters in their default values. \revision{For each experiment, we run the tested approaches on 10 different  generated problem instances and  report their average total revenues and average batch processing times for a whole day (from 00:00:00 to 23:59:59).} All our experiments are conducted on an Intel Xeon X5675 CPU @3.07 GHZ with 32 GB RAM in Java. The code of our queueing-theoretic vehicle dispatching framework and prediction methods can be accessed in our github project \cite{sourcecode}.

{\small
	\begin{table}[t!]
		\centering
		\caption{\small Results of the Estimated Idle Time}\label{tab:idle_estimate} \vspace{-2ex}
		\begin{tabular}{c|c|c|c}
			\#Drivers & MAE (s) & RMSE (\%) & Real RMSE (s) \\\hline\hline
			1K & 2.12 &  5.02 & 8.73 \\
			2K   & 1.89 & 4.76 & 6.89\\
			3K  & 1.78 & 4.53 & 4.43 \\
			4K & 2.04 & 5.11 & 7.04 \\
			5K & 2.22 & 5.47 & 11.24 \\
			6K & 2.54 & 5.93 & 13.81 \\
			7K & 3.20 & 6.45 & 26.39 \\
			8K & 4.34 & 7.43 & 44.43 \\
			\hline
		\end{tabular}
	\end{table}
}

\subsection{Results of the Estimated Idle Time}

In this section, we evaluate the accuracy of our queueing theoretic model on estimating the idle time of the drivers after finishing their assigned tasks. To show the results, we vary the number of drivers from 1K to 8K and keep the other parameters in their default values as shown in Table \ref{tab:settings}. We report the mean average error (MAE), relative root mean square error (RMSE) and real root mean square error  (Real RMSE) of our estimated waiting time periods of drivers compared with their real waiting time periods in Table \ref{tab:idle_estimate}. 

From the results, we find that our queueing theoretic model can achieve good estimated idle time periods of the drivers after finishing their assigned tasks. When the number of drivers increases from 1K to 8K, the MAE, RMSE and real RMSE first decrease then increase. The reason is that our default batch interval is 3 seconds, when the number of drivers is 1K, the drivers can almost immediately receive new task after they finish their assigned tasks. However, due to the batch process, they need to wait until next batch process, which in fact leads to the major difference between the estimated waiting time periods and the real ones. When the number of drivers increases from 1K to 4K, more and more drivers needs to wait for a while to receive a new task after they finish their last tasks. Then the estimation errors caused by  the batch processes become tiny. When the number of drivers continues increasing from 4K to 8K, the idle time of drivers also increases obviously. The MAE and real RMSE of the results of our queueing theoretic model also increases obviously, however the relative RMSE  only increases 2.32\%, which shows that our estimation model is accurate. \revision{Figure \ref{subfig:idle_p} shows the predicted idle times for each region achieved by our queueing theoretic model, which is very close to the real idle times of drivers (shown in Figure \ref{subfig:idle_r})  during the running of our vehicle dispatching framework.}

\begin{figure}[t!]\centering
	\subfigure[][{\small Predicted Idle Time}]{
		\scalebox{0.032}[0.032]{\includegraphics{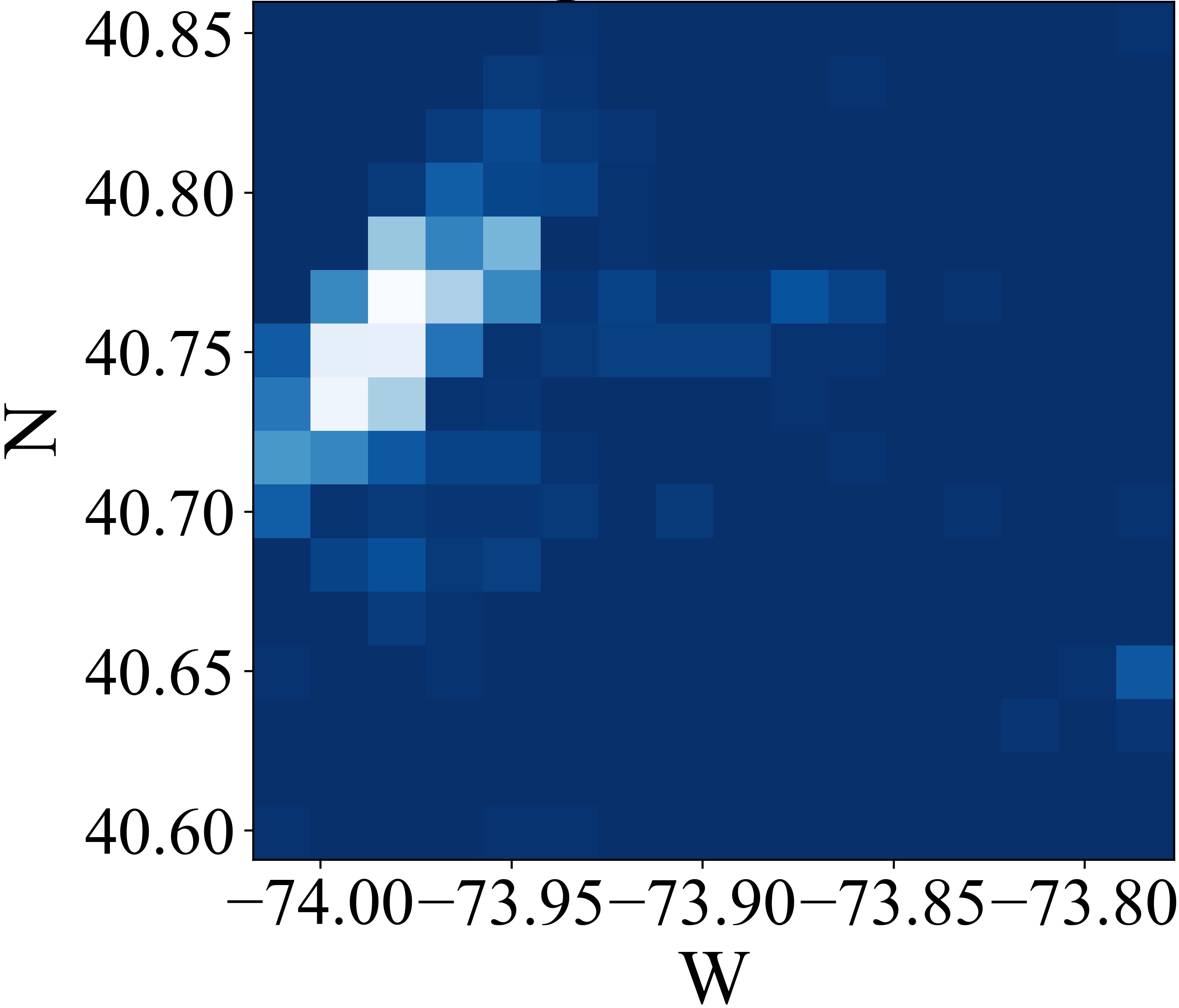}}
		\label{subfig:idle_p}}
	\subfigure[][{\small Real Idle Time}]{
		\scalebox{0.032}[0.032]{\includegraphics{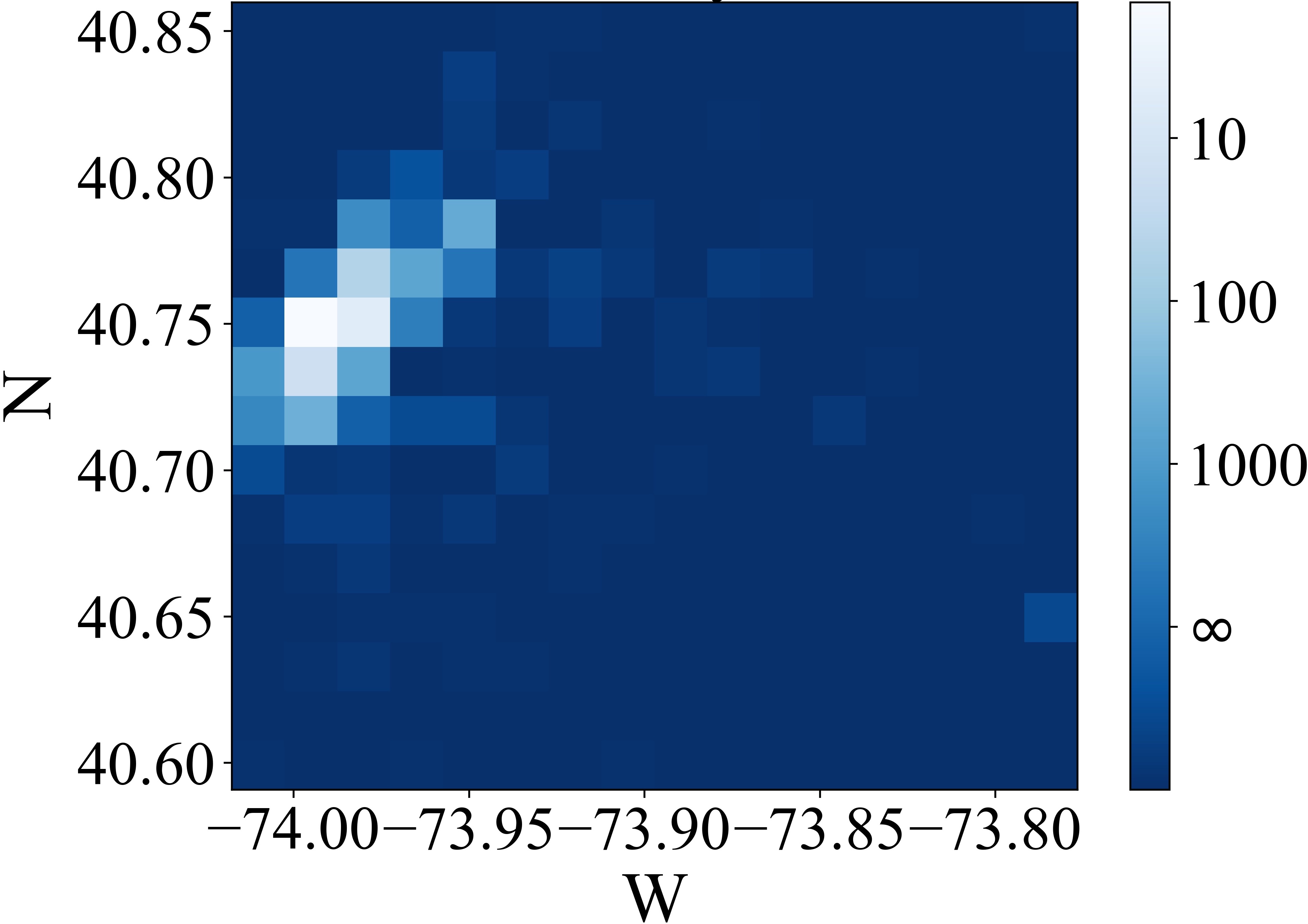}}
		\label{subfig:idle_r}}
	\vspace{-3ex}
	\caption{\small Comparison of Predicted and Real Idle Time.}
	\label{fig:prediction_comp} 
\end{figure}

\revision{
\subsection{Effects of the  Prediction Methods}
In this section, we evaluate the effects of  prediction methods for three prediction related  approaches, POLAR, IRG and LS.  Table \ref{tab:effect_prediction} shows the achieved total revenue of three prediction related approaches with default parameters (in Table \ref{tab:settings}) by using different prediction methods (introduced in Section \ref{sec:demand_prediction}). From the results, we can find that: a) the more accurate the prediction method is, the higher total revenue that each approach can achieve; b) LS is the best approach on utilizing the prediction information to improve the total revenue. 

{\small
\begin{table}[h!]
	\centering
	\caption{Results of Effects of Prediction Methods ($10^8$)}\label{tab:effect_prediction} \vspace{-2ex}
	\begin{tabular}{c|c|c|c|c|c}
			  & HA & LR&GBRT&DeepST &Real \\\hline\hline
		IRG   & 2.2460 & 2.3203 & 2.3446 & 2.3756 &2.3899\\
		LS   & 2.2921 & 2.3725 &2.4267 &2.4625 &2.4727\\
		POLAR  & 2.0460 & 2.2293 & 2.2767& 2.2953 &2.3285\\
		\hline
	\end{tabular}
\end{table}
}
}

\subsection{Experimental Results of Vehicle Dispatching Approaches}

In this section, we show the effects of the number, $n$, of drivers, the base pickup waiting time $\tau$, the length, $\Delta$, of batch interval, and the length, $t_c$, of time window to estimate the arrival rates of riders and rejoined drivers.

\noindent \textbf{Effect of the Number, $n$, of Drivers.} Figure \ref{fig:n_driver} illustrates the experimental results on varying the number of drivers from 1K to 5K, where other parameters are in their default values. In Figure \ref{subfig:n_driver_score}, when the number of drivers increases from 1K to 5K, all the tested approaches can achieve results with increasing total revenue. The reason is that when more drivers are available, more riders can be served before their pickup deadlines. When the number of drivers is 1K, our IRG and LS approaches can achieve higher total revenue than RAND, LTG, NEAR and POLAR. The difference between the results of our IRG and LS are small. When the number of drivers increases, the advantage of our IRG and LS in terms of the total revenue become narrow. We also notice that when the number of drivers reaches 5K, all the tested approaches can achieve results with total revenue close to the upper bound. The reason is that when there are 5K drivers, almost all the riders can be served as long as he/she joins the platform. \revision{Our LS can achieve from 78.1\% to 92.0\% of the upper bound revenue when the number of drivers increases from 1K to 5K.} To clearly show the differences between the total revenues of our tested approaches, we will not plot out the results of UPPER as they are always same with the results in Figure \ref{subfig:n_driver_score}. 
In Figure \ref{subfig:n_driver_cpu}, when the number of drivers increases, the batch running time of all the tested approaches also increases slightly, which is because in each batch there are more drivers  requiring more time to process. We can see that all the tested approaches can finish each batch processing within 2 seconds, which is unnoticeable to the users and acceptable for the batch processes with 3-second intervals.

\begin{figure}[t!]\centering
	\subfigure{
		\scalebox{0.12}[0.12]{\includegraphics{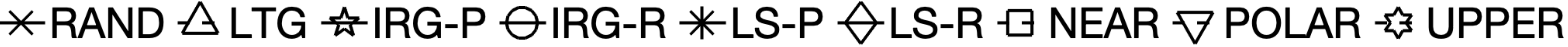}}}\hfill\\
	\addtocounter{subfigure}{-1}\vspace{-2ex}
	\subfigure[][{\small Total Revenue}]{
		\scalebox{0.2}[0.2]{\includegraphics{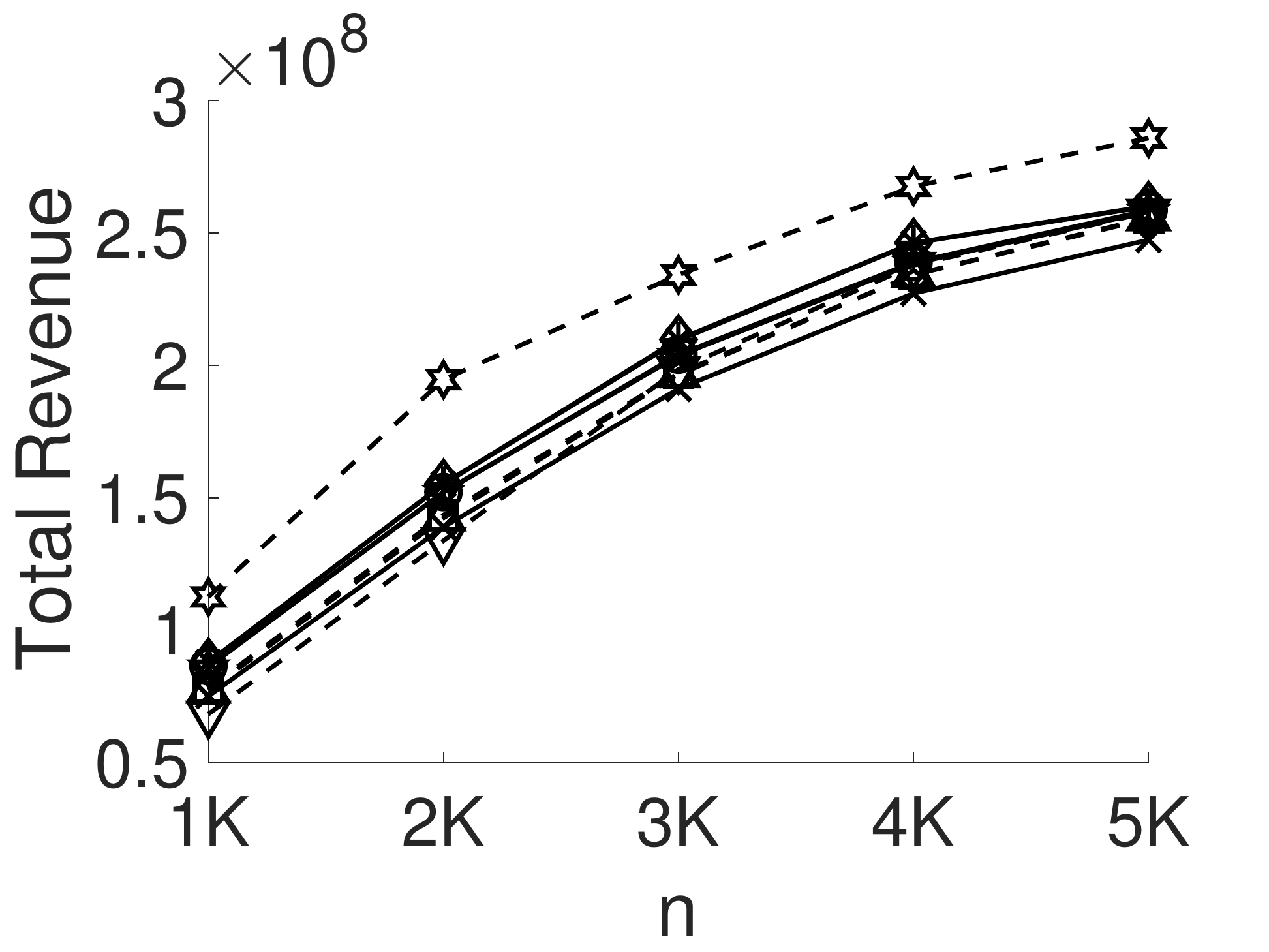}}
		\label{subfig:n_driver_score}}
	\subfigure[][{\small Batch Running Time}]{
		\scalebox{0.2}[0.2]{\includegraphics{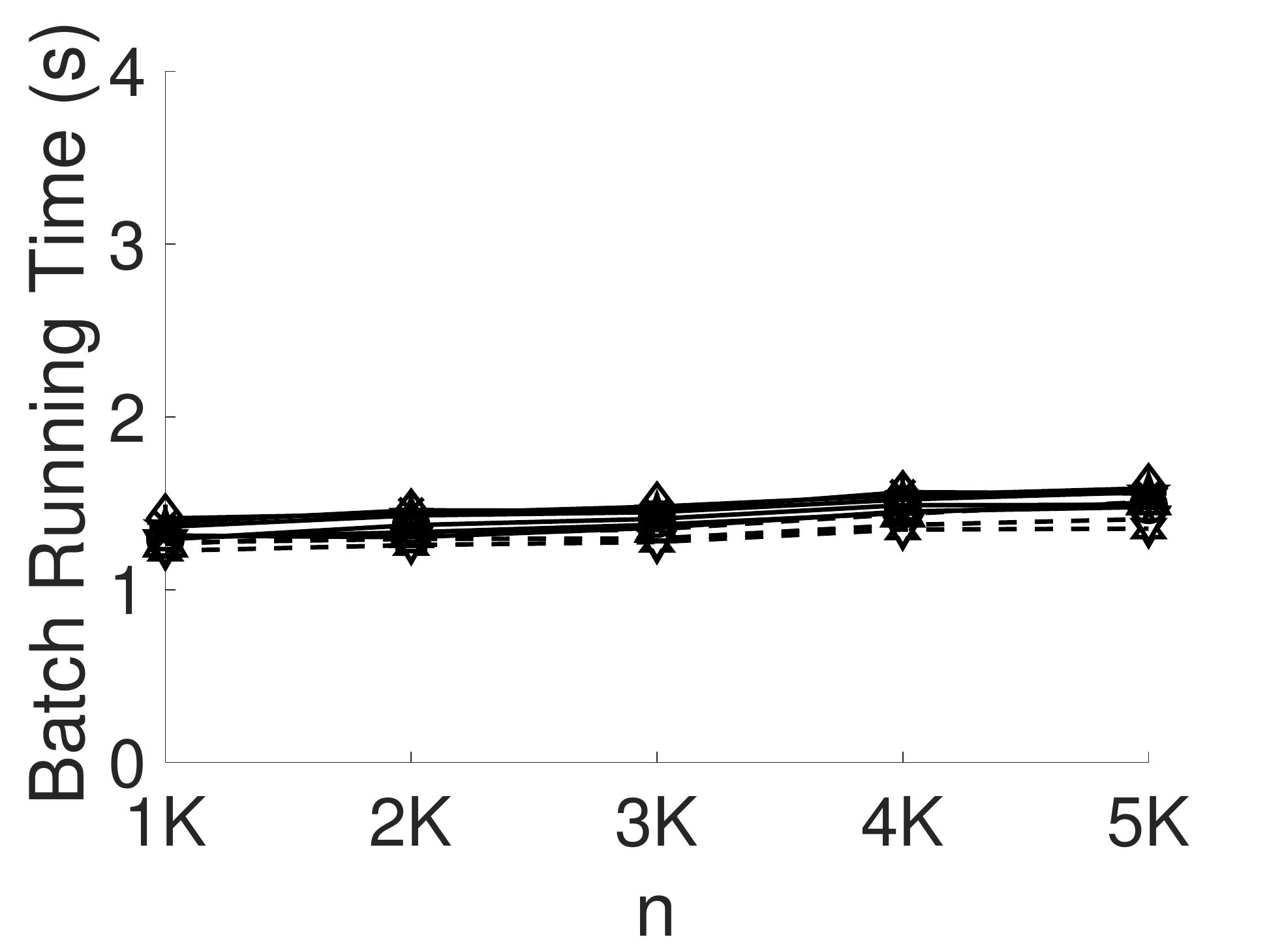}}
		\label{subfig:n_driver_cpu}}\vspace{-2ex}
	\caption{\small Effect of Number of Drivers $n$.}
	\label{fig:n_driver} 
\end{figure}

\begin{figure}[t!]\centering
	\subfigure{
		\scalebox{0.12}[0.12]{\includegraphics{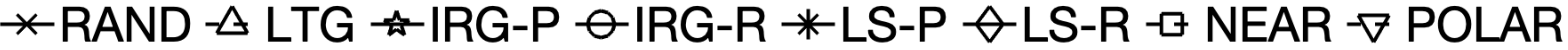}}}\hfill\\
	\addtocounter{subfigure}{-1}\vspace{-2ex}
	\subfigure[][{\small Total Revenue}]{
		\scalebox{0.2}[0.2]{\includegraphics{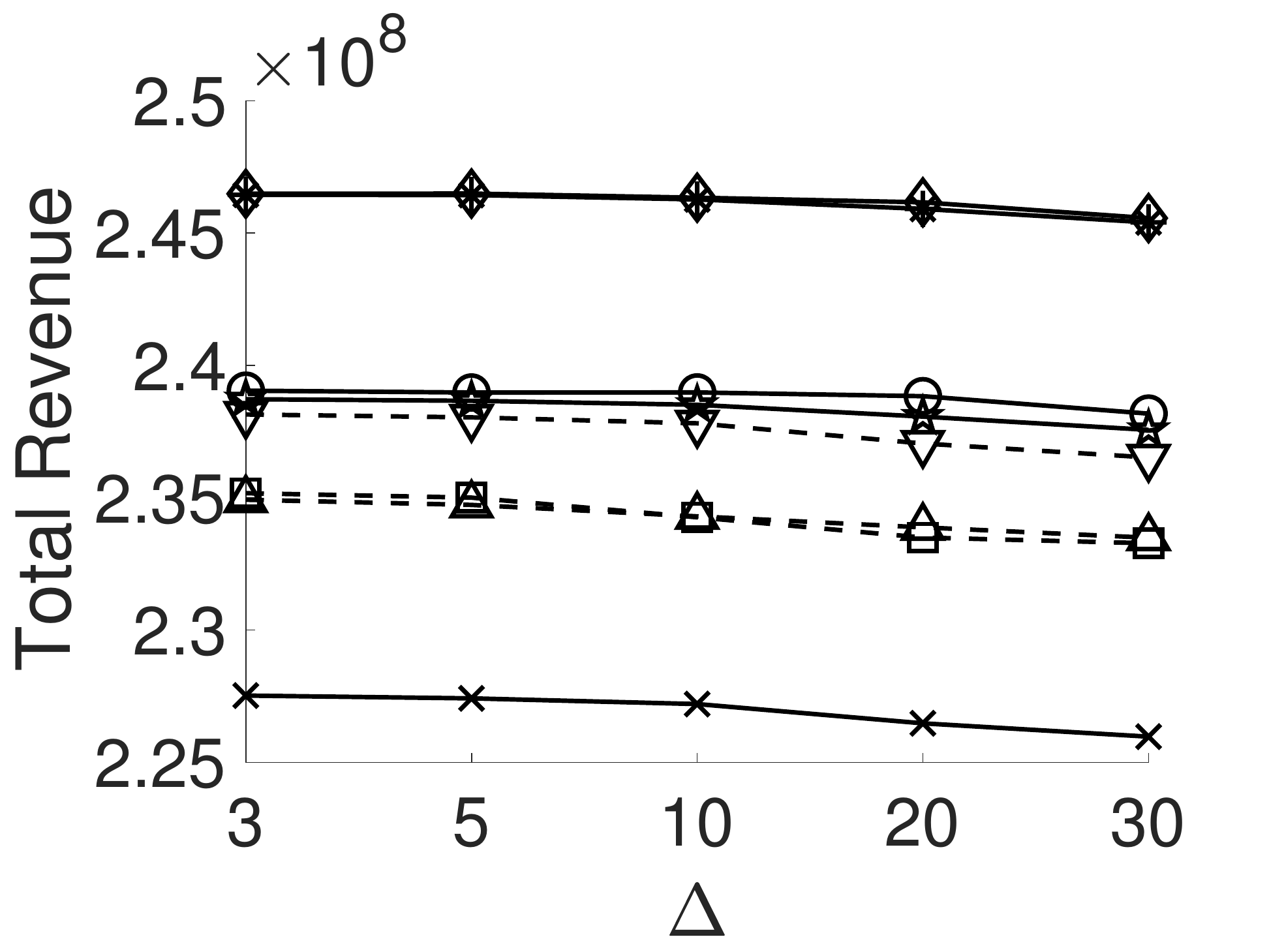}}
		\label{subfig:frame_score}}
	\subfigure[][{\small Batch Running Time}]{
		\scalebox{0.2}[0.2]{\includegraphics{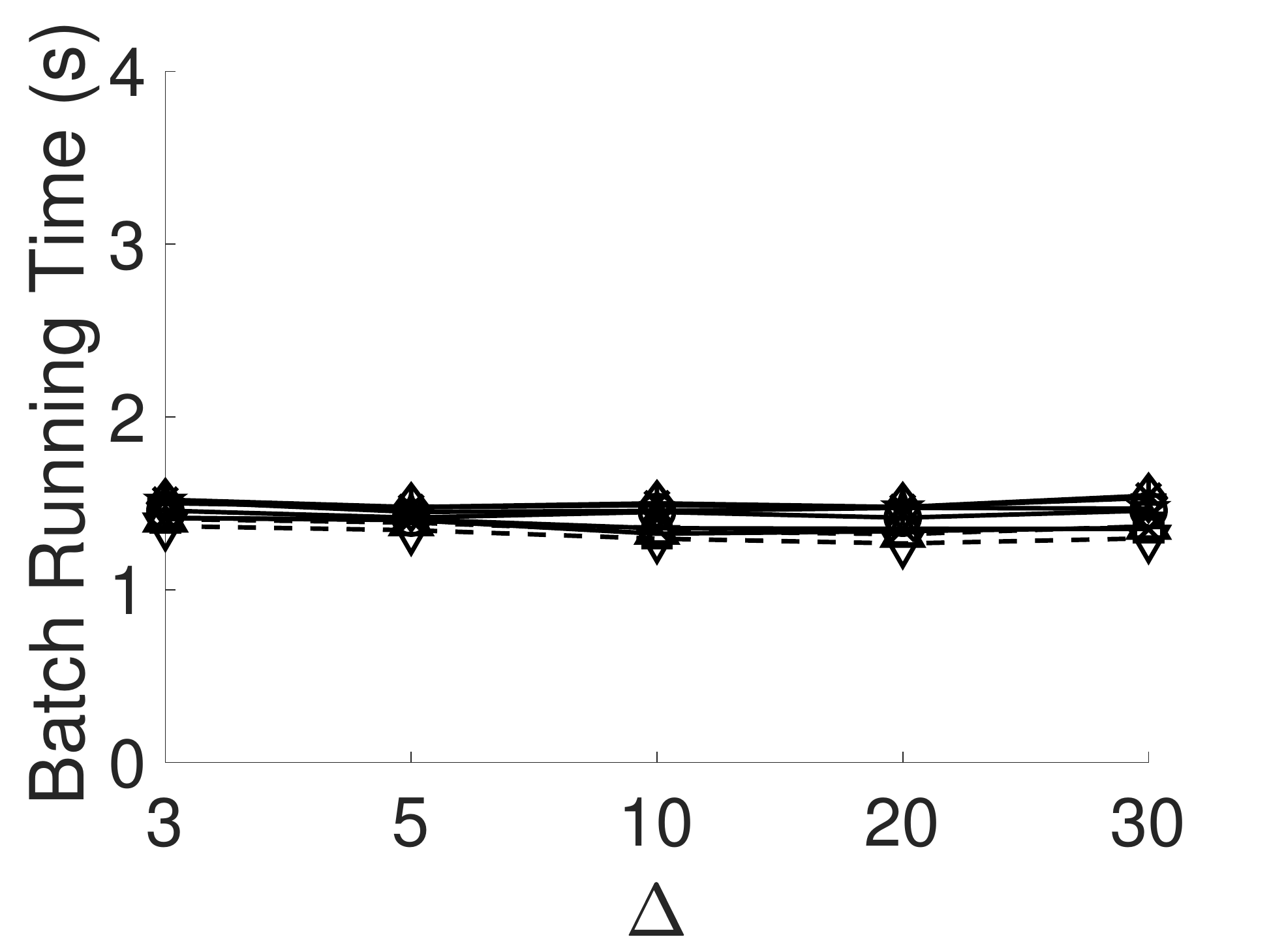}}
		\label{subfig:frame_cpu}}\vspace{-2ex}
	\caption{\small Effect of Batch Length $\Delta$.}
	\label{fig:frame} \vspace{-1ex}
\end{figure}

\noindent \textbf{Effect of the Length, $\Delta$, of Batch Interval.} Figure \ref{fig:frame} shows the experimental results on varying the length, $\Delta$, of the batch interval from 3 to 30 seconds, while other parameters are set to their default values. As shown in Figure \ref{subfig:frame_score}, when the length, $\Delta$, of batch interval increases from 3 to 30 seconds, the total revenues of the results achieved by the tested approaches decrease slightly. The reason is that when the length of the batch interval increases, more riders may be missed before their pickup deadlines within two consecutive batches. In other words, when $\Delta$ increases, the probability of a rider becomes time out will increase during the batch intervals, when the platform does not respond to any riders or drivers. Another reason is that when drivers become available, they also need to wait for the next batch to be assigned with new riders, which also leads to the bad effect on the total revenue. Thus, in real applications, $\Delta$ should not be too large. In addition, we notice that our IRG-P and LS-P can achieve higher total revenues than RAND, LTG, NEAR and POLAR. We find that when we use the ground truth of the taxi demand for our IRG-R and LS-R algorithms, they can achieve higher total revenues than IRG-P and LS-P, which shows the importance of the accuracy of the taxi demand methods. In other words, for the real applications, a more accurate prediction model can bring increases on the total revenue. 
In Figure \ref{subfig:frame_cpu}, the batch running time of the tested approaches slightly increases, since the number of riders and drivers for each batch will increase when $\Delta$ increases. 

\begin{figure}[t!]\centering
	\subfigure{
		\scalebox{0.12}[0.12]{\includegraphics{bar_without_upper-eps-converted-to.pdf}}}\hfill\\
	\addtocounter{subfigure}{-1}\vspace{-2ex}
	\subfigure[][{\small Total Revenue}]{
		\scalebox{0.2}[0.2]{\includegraphics{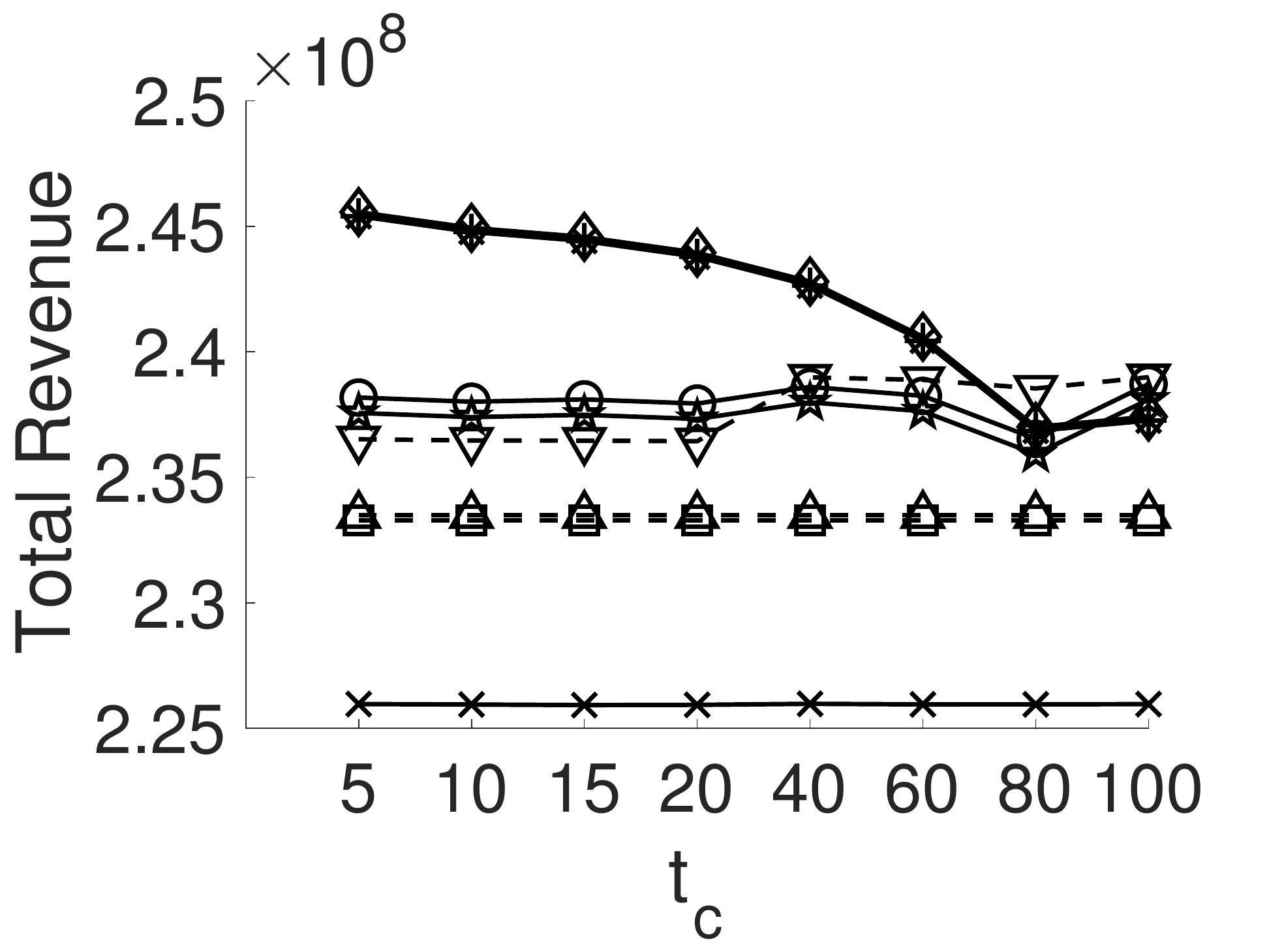}}
		\label{subfig:range_score}}
	\subfigure[][{\small Batch Running Time}]{
		\scalebox{0.2}[0.2]{\includegraphics{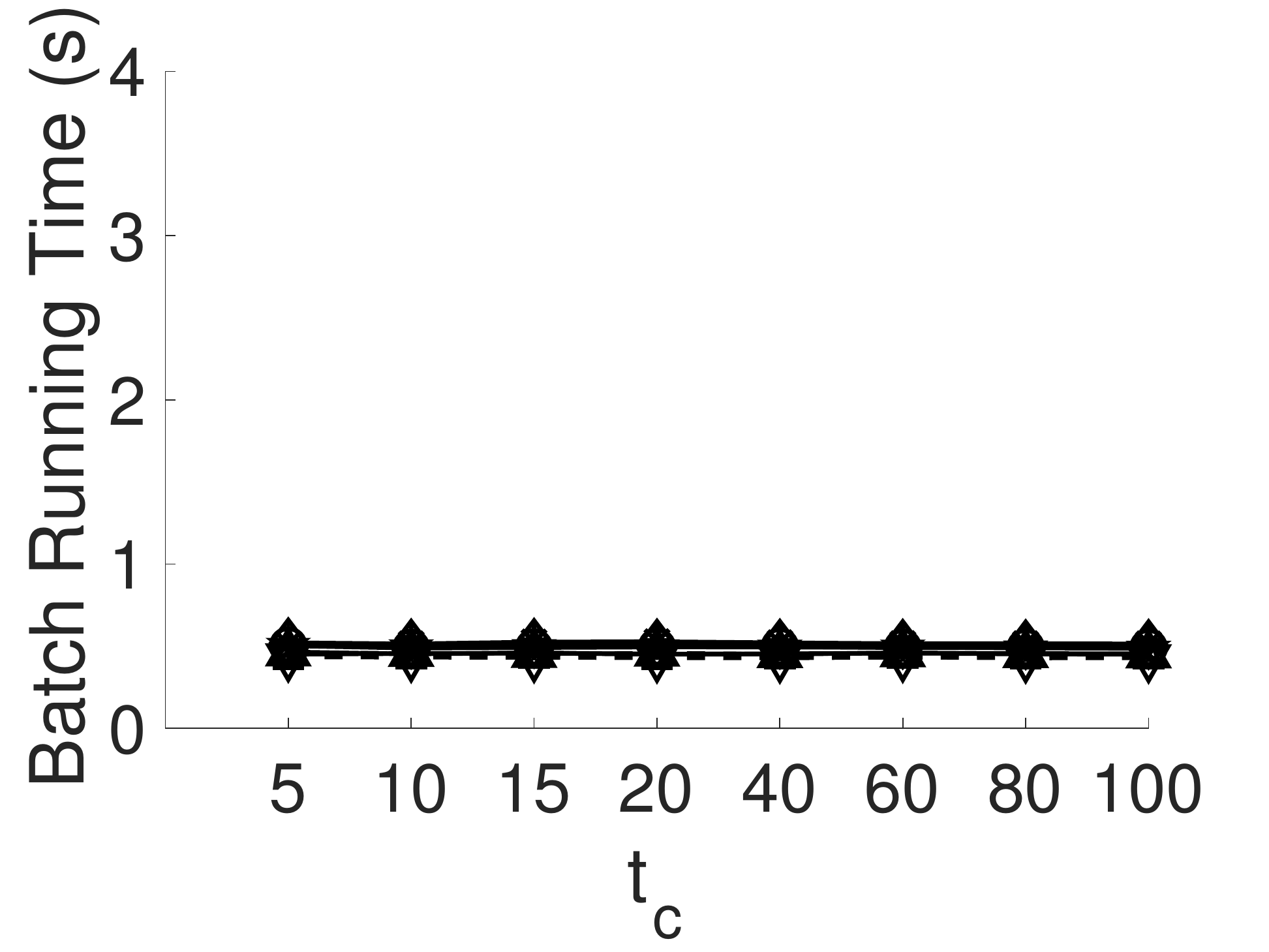}}
		\label{subfig:range_cpu}}\vspace{-2ex}
	\caption{\small Effect of Time Window $t_c$.}
	\label{fig:range} 
\end{figure}
\begin{figure}[t!]\centering
	\subfigure{
		\scalebox{0.12}[0.12]{\includegraphics{bar_without_upper-eps-converted-to.pdf}}}\hfill\\ 
	\addtocounter{subfigure}{-1}\vspace{-2ex}
	\subfigure[][{\small Total Revenue}]{
		\scalebox{0.2}[0.2]{\includegraphics{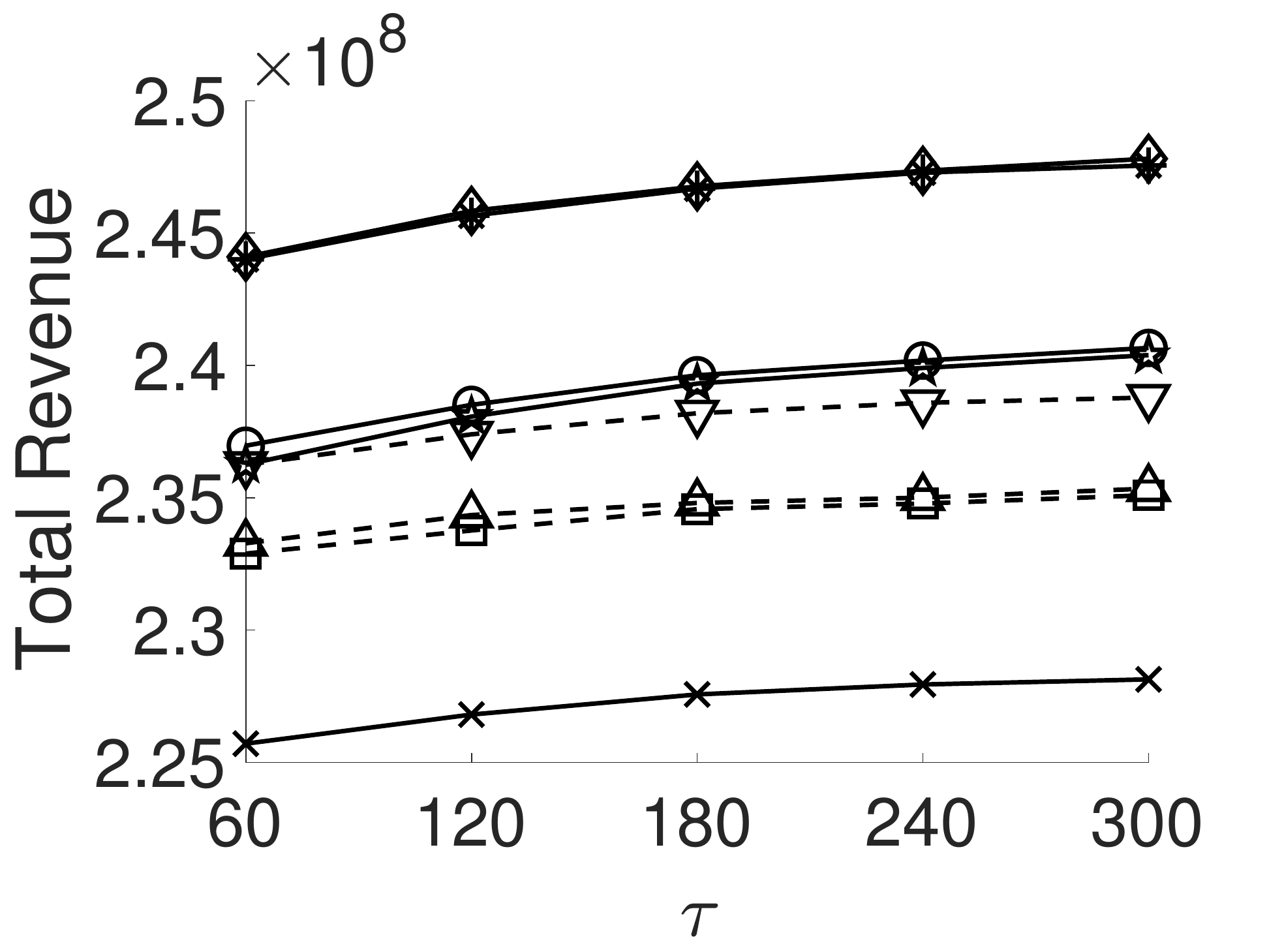}}
		\label{subfig:waitingtime_score}}
	\subfigure[][{\small Batch Running Time}]{
		\scalebox{0.2}[0.2]{\includegraphics{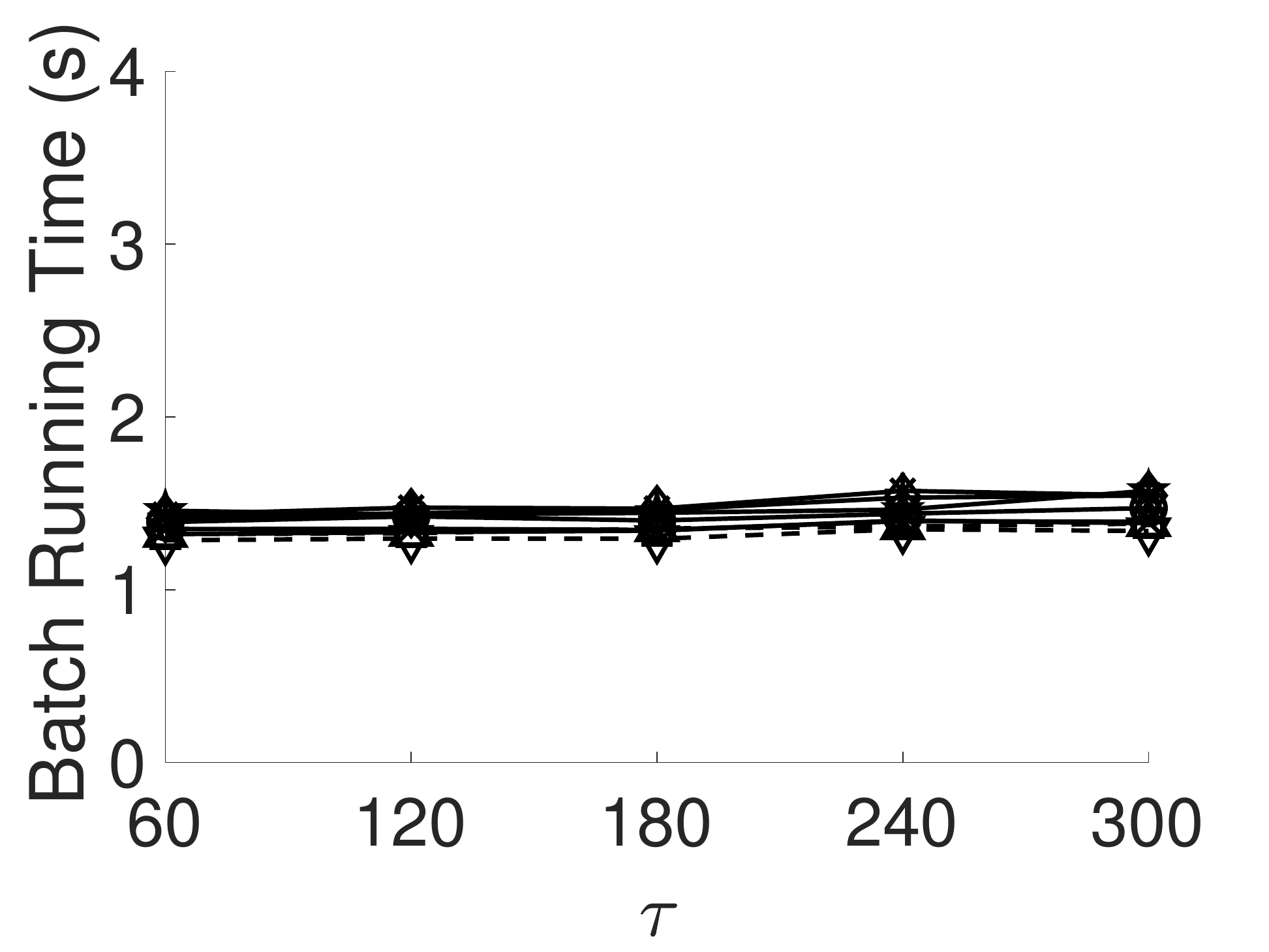}}
		\label{subfig:waitingtime_cpu}}\vspace{-2ex}
	\caption{\small Effect of Base Waiting Time $\tau$.}
	\label{fig:waitingtime} \vspace{-1ex}
\end{figure}

\noindent \textbf{Effect of the Length, $t_c$, of Time Window.} Figure \ref{fig:range} presents the experimental results on varying the length $t_c$ of time window on estimating the arrival rate of new riders and rejoined drivers. In Figure \ref{subfig:range_score}, the total revenue achieved by IRG and LS will decrease when $t_c$ becomes larger than 20 minutes. The reason is that most taxi trips in NYC taxi trip dataset have a travel time of less than 20 minutes \cite{cheng2017utility}. The effect of future rejoined drivers in more than 20 minutes later is almost neglectable for our IRG and LS algorithms. However, when $t_c$  becomes 40 minutes, POLAR can perform better than  itself in the experiment of $t_c$ smaller than 20 minutes.  Thus, in real platform, $t_c$ should not be too large. Since RAND and LTG do not consider the demand and supply of the taxis in future, the length, $t_c$, of time window has no effect on them.  In Figure \ref{subfig:range_cpu}, $t_c$ has no clear effect on the running time of our tested approaches.

\noindent \textbf{Effect of the Base Waiting Time $\tau$.} Figure \ref{fig:waitingtime} illustrates the effect of the waiting time $\tau$ of riders by varying $\tau$ from 60 to 300 seconds while keeping other parameters in their default values. In Figure \ref{subfig:waitingtime_score}, when the waiting time of riders $\tau$ increases, the total revenue of the results achieved by our tested approaches also increases. The reason is that when riders can wait for longer time, the probability that they can be served by some drivers will increase, which is consistent with human intuition. With the help of ground truth of the taxi demand (more accurate than our predicted demand), LS-R can achieve slightly higher total revenue than LS-P.  IRG, LS and their variants  can all surpass RAND, LTG, NEAR and POLAR. In Figure \ref{subfig:waitingtime_cpu}, the batch running time of tested approaches increases slightly when the waiting time of riders increases. The reason is that when riders can wait for longer time and the number of drivers does not change, the number of riders in each batch will also increase, which leads to the processing time of each batch becomes longer.

In summary, LS and IRG can perform better than RAND,  LTG, NEAR and POLAR in terms of total revenue. Our proposed algorithms are more effective when the number of drivers is smaller (e.g., 1K drivers our in experiments). The accuracy of taxi demand prediction method can affect the final results on the total revenue. Thus, taxi demand prediction models with higher accuracy are more valuable for the platform. Our framework is efficient. In all the experiments, the running time of each batch for all the tested approaches is less than 2 seconds, which is affordable for the platform to perform a batch process with 3 seconds for each batch interval.
\section{Related Work}
\label{sec:related}

Recently, online car-hailing platforms develops rapidly,  which has drawn attention from  academia and industry.

Our MRVD problem is related to task assignment in spatial crowdsourcing \cite{kazemi2012geocrowd, cheng2015reliable, cheng2016task, tong2017flexible}, which assign a set of workers to the locations of tasks to conduct subject to various constraints and optimization goals. However, in our MRVD problem, each order has a pickup location and a destination, while each task in spatial crowdsourcing usually has only one required location. In \cite{kazemi2012geocrowd}, based on the publishing models, the authors classified the spatial crowdsourcing in two modes: worker selected task (WST) mode  \cite{deng2013maximizing} and server assigned tasks (SAT) mdoe \cite{cheng2015reliable, cheng2016task, tong2017flexible}. In WST mode, workers select tasks by themselves. In SAT mode, the server/system has the control on assign tasks to workers base on its objectives. In SAT mode, there are two processing styles: online task assignment mode  \cite{tong2017flexible} and batch-based task assignment mode \cite{kazemi2012geocrowd, cheng2015reliable, cheng2016task}. Recently, researchers start to utilize the prediction models to predict the future distributions of workers and tasks to improve the overall performance in a relatively long time period (e.g., 1 day). For instance, researchers build an offline blueprint based on the predicted distributions of workers and tasks, then use it to guide the online task assignment to maximize the total  number of assigned tasks~\cite{tong2017flexible}. Our MRVD targets on maximizing the total revenue of the platform, which cannot apply existing solutions directly. Thus, we develop our queueing theoretic framework, which uses queueing theory to estimate the idle time of drivers based on the predicted number of orders and drivers in each region.

Our MRVD problem is also related to dial-a-ride problem (DARP), which assume a fleet of vehicles located at a common depot, and schedules should be made to accommodate $m$ rider requests based on their pick-up and drop-off time constraint. Existing works on DARP have mainly focused on static offline DARP, where the constraints are known beforehand. The general DARP is NP-hard and intractable, unless its scale is not big (e.g., hundreds of vehicles  and riders) \cite{cordeau2006branch}. \cite{cordeau2003tabu} uses a heuristic method called tabu search to find the neighbourhood solution from current solution, to avoid finding cycle result and local optimum, they forbid the recent visited answers and use some diversification mechanism.

With the emergence of ridesharing business, many riders prefer choosing the ridesharing service, as it is cheaper than non-share car request with limited time delay. The authors \cite{cheng2017utility} designs an algorithm to dispatch the similar rider to the same car with a goal of maximizing the total utility, which includes the rider related utility, vehicle-related utility and trajectory-related utility. In \cite{zheng2018order}, the authors propose a packing-based approach, which first packs the riders together then assigns groups of riders to vehicles. To solve the scheduling problem for a vehicle with a set of assigned riders, authors in \cite{tong2018unified} propose a linear time complex method. However, ridesharing mainly focuses on scheduling and solving conflicts of route-sharable riders to vehicles, which is different from MRVD.

In addition, traffic prediction is also a critical technology in urban city transportation scenario. With accurate prediction, we can foresee the future and make plan to fulfill the long time revenue. There are many models which focusing on predicting the number of orders in the next time slot by integrating temporal and spatial information. \cite{zhang2017deep} proposes a Deep ST model which combines the geographical and historical traffic data together, to decrease the difference between estimated traffic flow number and actual count. With the powerful deep convolutional neural network and rich daily meta data (e.g., holiday and weather), they get the state of art prediction results.
\section{Conclusion}
\label{sec:conclusion}

In this paper, we study the problem of maximum revenue vehicle dispatching problem (MRVD), in which rider requests dynamically arrive and drivers need to serve as many riders as possible such that the entire revenue of the platform is maximized. We prove that the MRVD problem is NP-hard and intractable. Through analyses, we find to maximize the total revenue, we need to give higher priorities to ride orders with long travel cost and less idle time. We propose a queueing-theoretic framework, which predicts the taxi demand (rider orders) offline and schedule the drivers to regions where the idle time of them will be small. Our framework dispatching drivers to riders in a batch-based processing for every $\Delta$ seconds. To handle the batch vehicle dispatching problem, we propose two heuristic approaches, namely idle ratio oriented greedy (IRG) and local search (LS). Through experiments on the real and synthetic data sets, we show the effectiveness and efficiency of our queueing-theoretic vehicle dispatching framework.

\begin{acks}
Lei Chen's work is partially supported by  National Key Research and Development  Program of China Grant No. 2018AAA0101100, the Hong Kong RGC GRF Project 16202218, CRF Project C6030-18G, C1031-18G, C5026-18G, AOE Project AoE/E-603/18, Theme-based project TRS T41-603/20R, China NSFC No. 61729201, Guangdong Basic and Applied Basic Research Foundation 2019B151530001, Hong Kong ITC ITF grants ITS/044/18FX and ITS/470/18FX, Microsoft Research Asia Collaborative Research Grant, HKUST-NAVER/LINE AI Lab, Didi-HKUST joint research lab, HKUST-Webank joint research lab grants. Peng Cheng's work is sponsored by Shanghai Pujiang Program 19PJ1403300. Xuemin Lin's work is supported by ARC  DP200101338. Libin Zheng's work is supported by the Fundamental Research Funds for the Central Universities, Sun Yat-sen University.
\end{acks}

\balance

\newpage

\bibliographystyle{ACM-Reference-Format}
\bibliography{../references/add}

\appendix

\section{Methods and Results of Offline Demand Prediction}
\label{sec:prediction}


To predict the future order number, we use different kinds of information, such as the order counts from previous time slots, and other meta data (e.g., time of day, day of week and city weather). We use the model of DeepST \cite{zhang2017deep}, which uses previous order numbers from three different time scales: \textit{closeness}, \textit{period}, \textit{trend}. Here, \textit{closeness} means the previous $N$ time slots; \textit{period} indicates the same time in previous $N$ days; \textit{trend} refers to the same time in previous $N$ weeks. Meanwhile, it also uses other features (e.g., weather information) to make a good prediction with Convolutional Neural Network \cite{krizhevsky2012imagenet}.

We use separated model to process three different features of categories, then add all the output together to get the final result. Meanwhile, we use day of week and time of day embedding feature from a lookup feature table.

\begin{figure*}[t!]\centering
	\subfigure[][{\small Results in Region 1 (7 A.M.)}]{
		\scalebox{0.21}[0.21]{\includegraphics{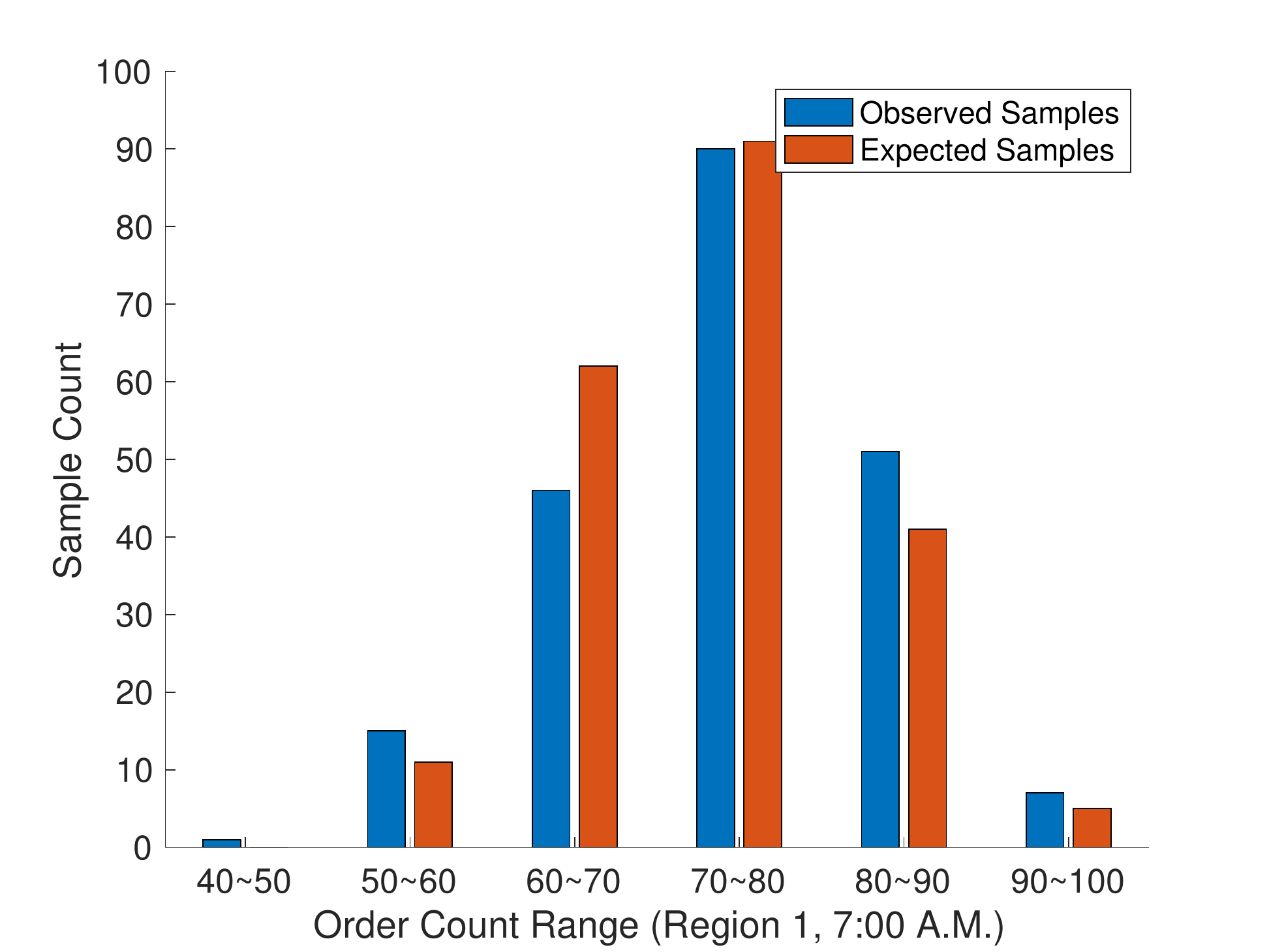}}
		\label{subfig:order_r1_7_bar}}
	\subfigure[][{\small Results in Region 1 (8 A.M.)}]{
		\scalebox{0.21}[0.21]{\includegraphics{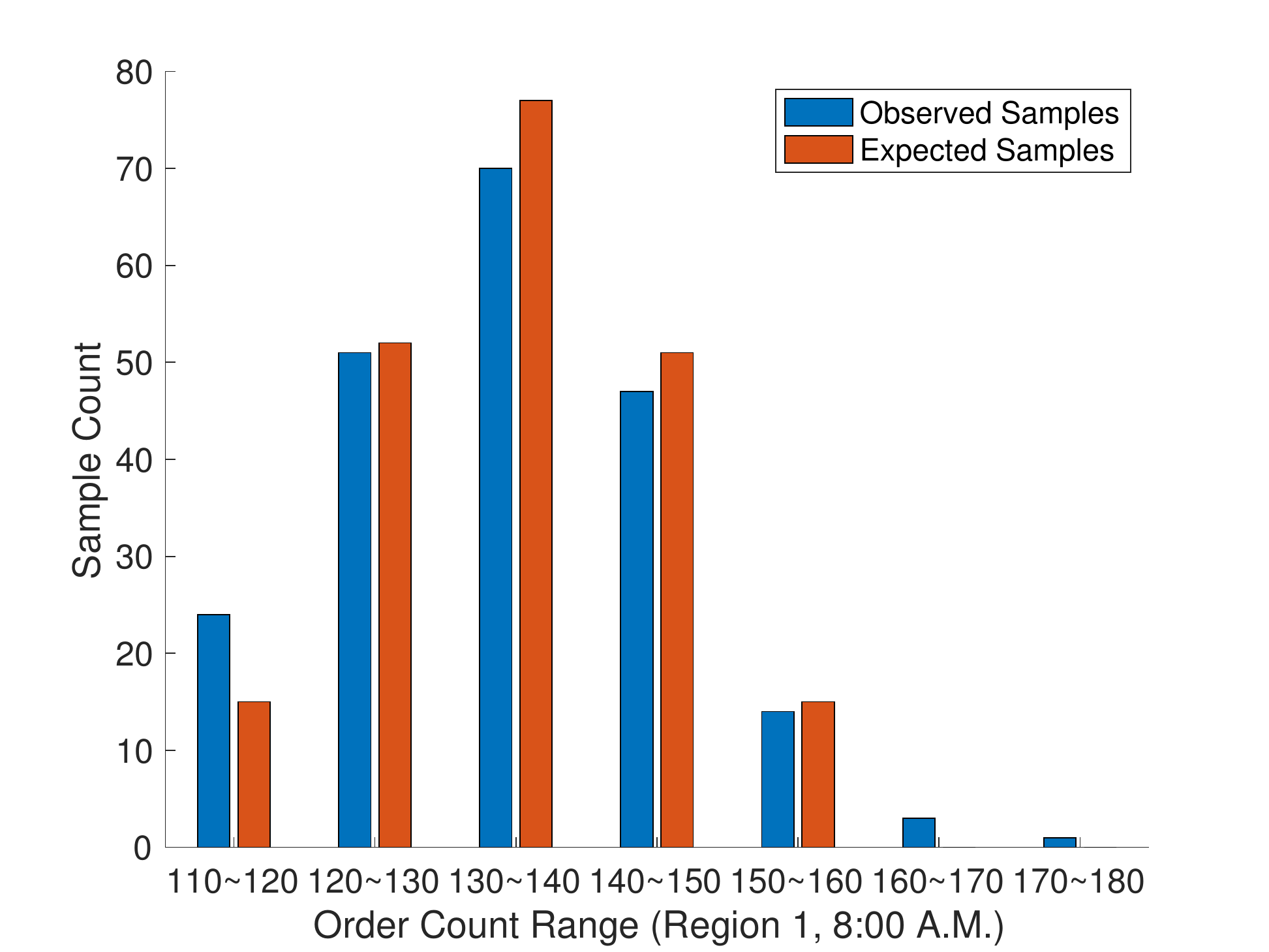}}
		\label{subfig:order_r1_8_bar}}
	\subfigure[][{\small Results in Region 2 (7 A.M.)}]{
		\scalebox{0.21}[0.21]{\includegraphics{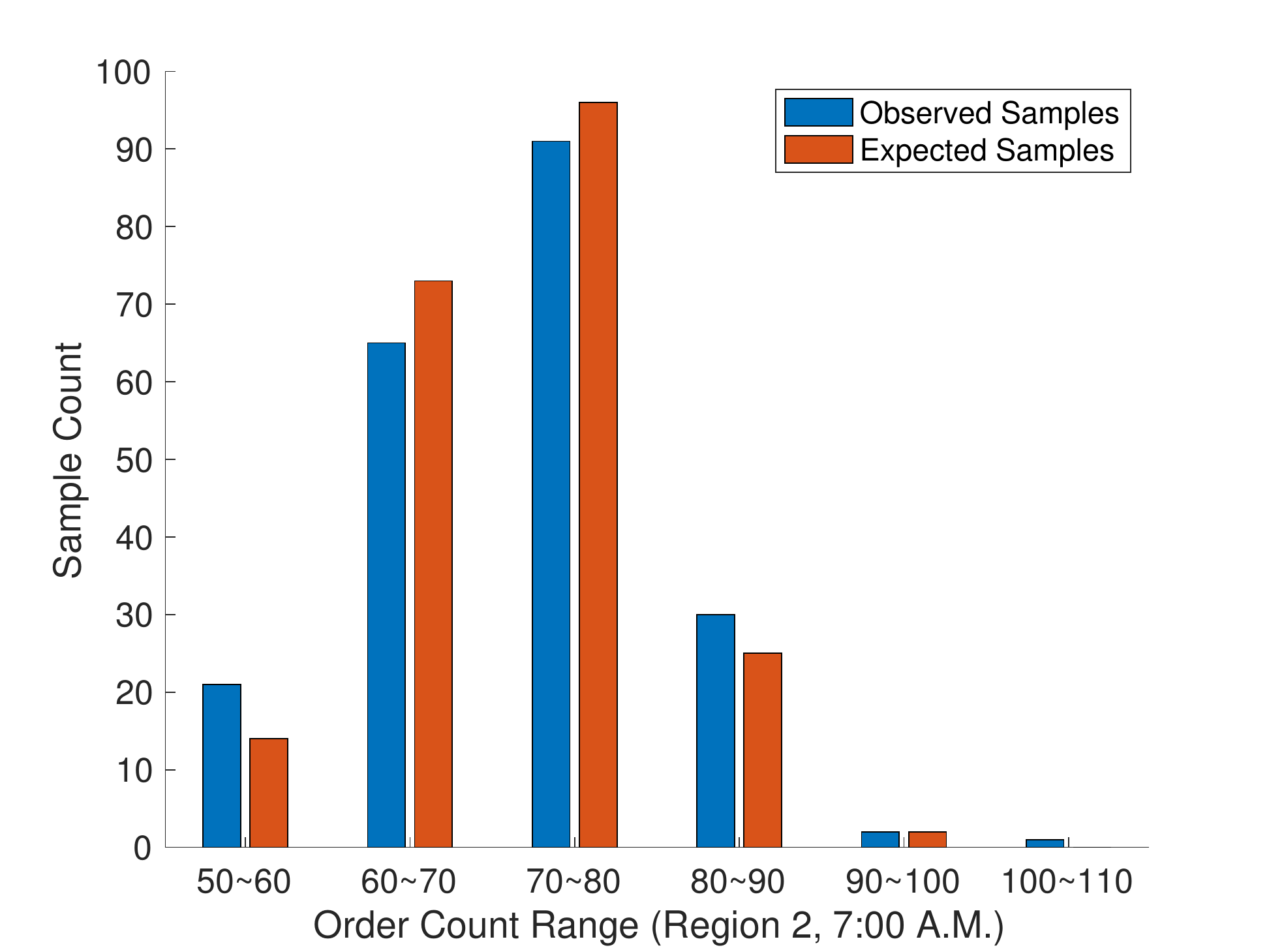}}
		\label{subfig:order_r2_7_bar}}
	\subfigure[][{\small Results in Region 2 (8 A.M.)}]{
		\scalebox{0.21}[0.21]{\includegraphics{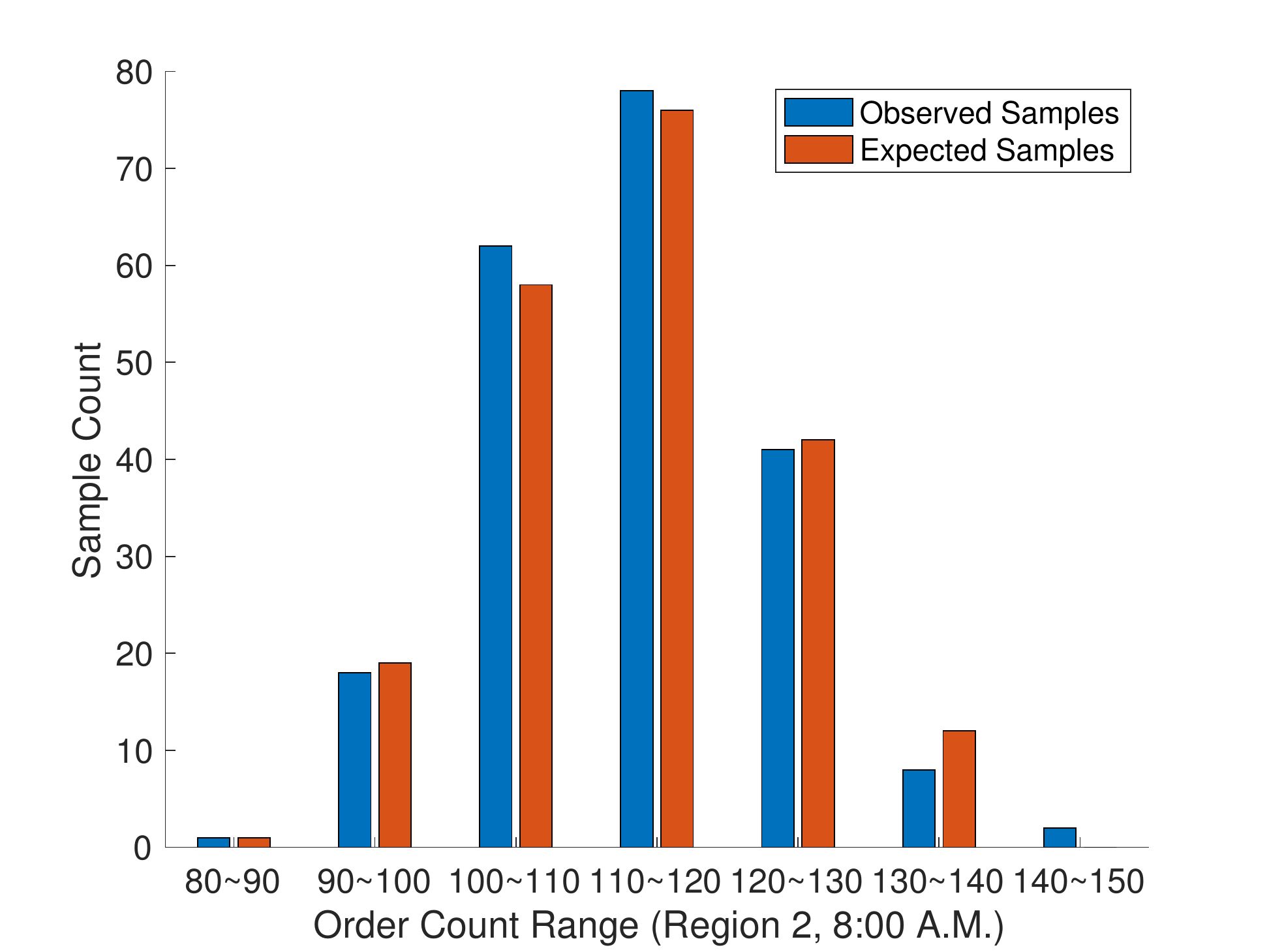}}
		\label{subfig:order_r2_8_bar}}\vspace{-1ex}
	\caption{\small Distribution of Order Quantity Samples in Different Regions and Time Periods.}
	\label{fig:order_poisson} 
\end{figure*}

\begin{figure*}[t!]\centering
	\subfigure[][{\small Results in Region 1 (7 A.M.)}]{
		\scalebox{0.21}[0.21]{\includegraphics{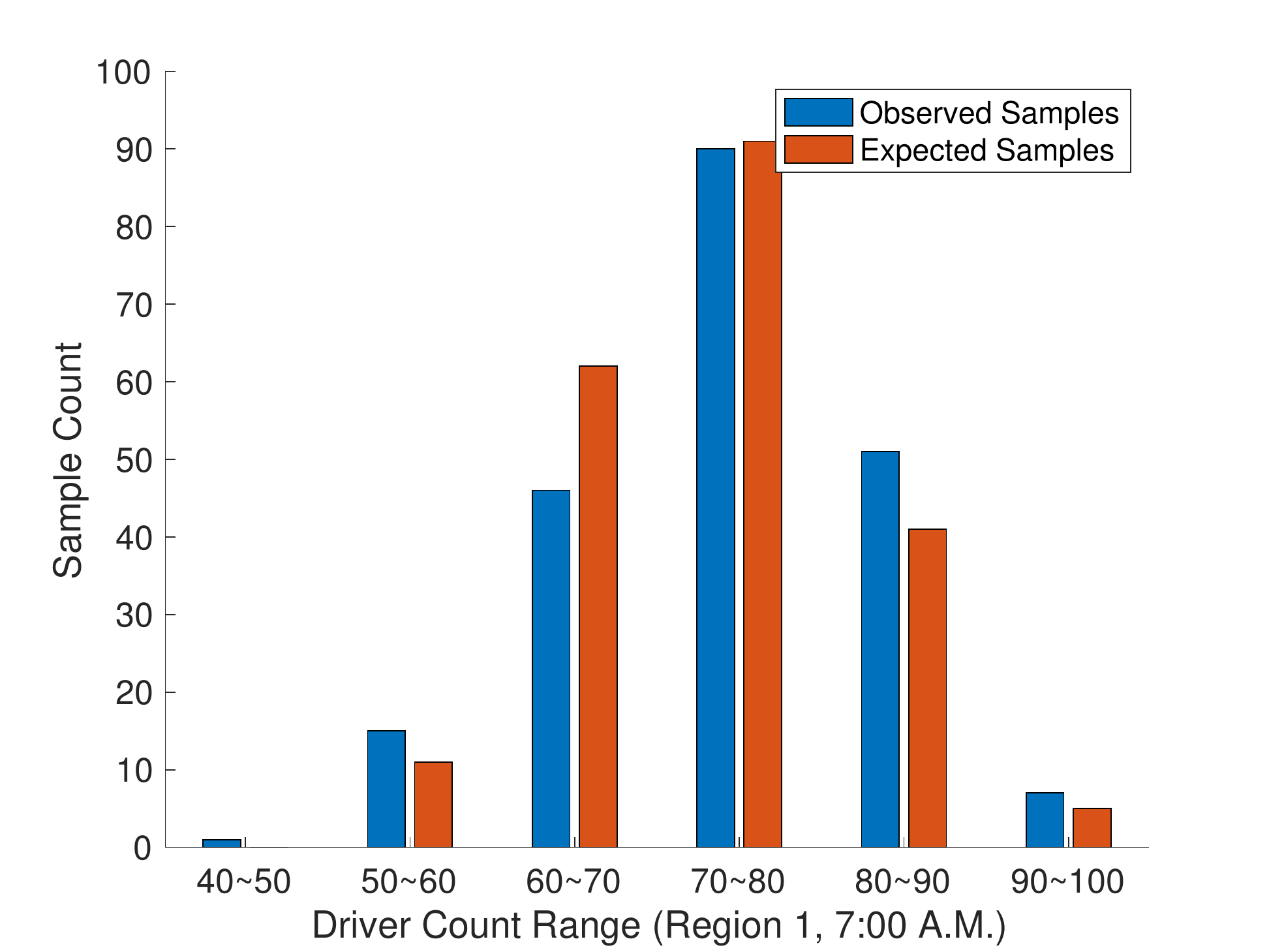}}
		\label{subfig:driver_r1_7_bar}}
	\subfigure[][{\small Results in Region 1 (8 A.M.)}]{
		\scalebox{0.21}[0.21]{\includegraphics{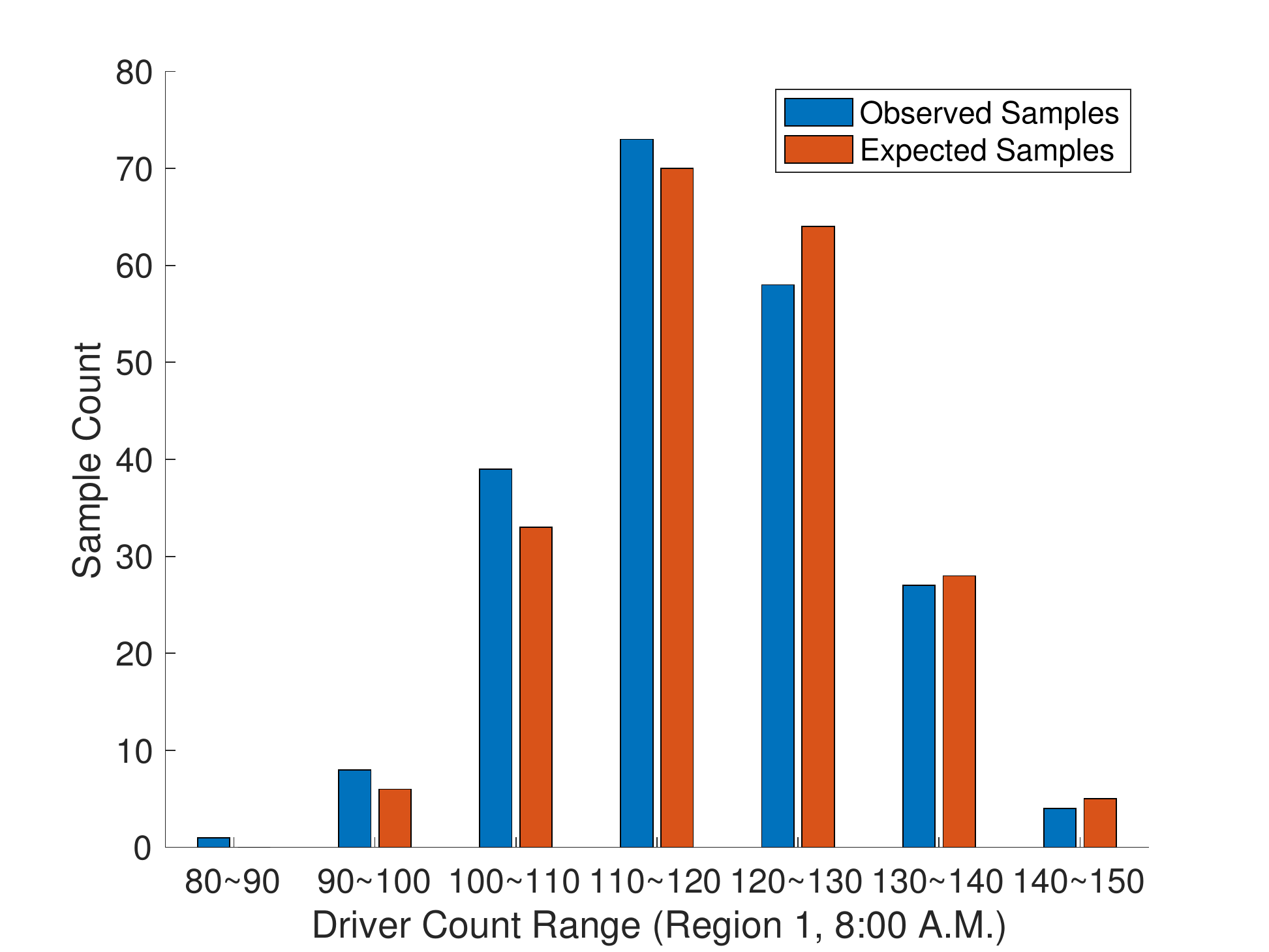}}
		\label{subfig:driver_r1_8_bar}}
	\subfigure[][{\small Results in Region 2 (7 A.M.)}]{
		\scalebox{0.21}[0.21]{\includegraphics{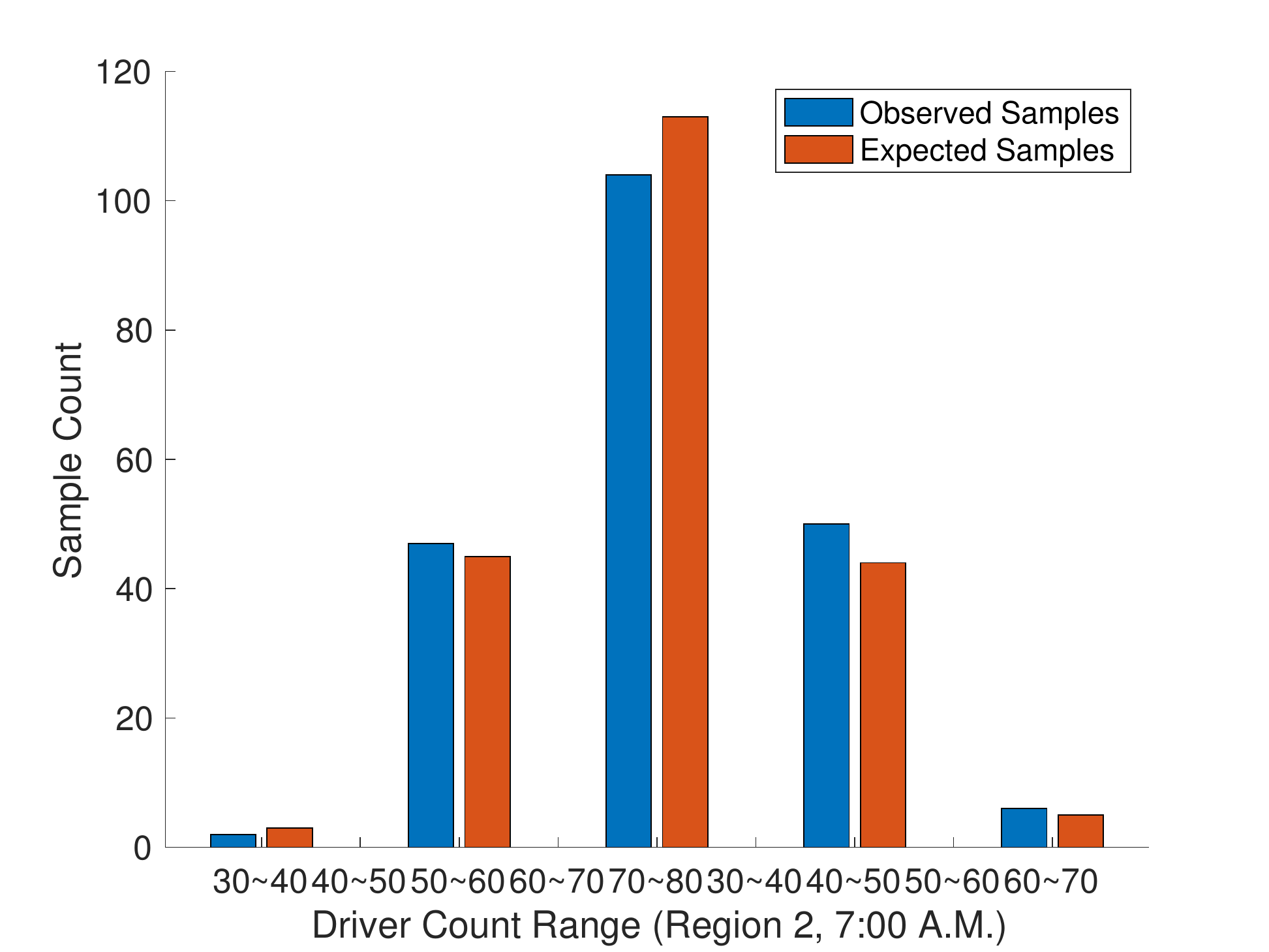}}
		\label{subfig:driver_r2_7_bar}}
	\subfigure[][{\small Results in Region 2 (8 A.M.)}]{
		\scalebox{0.21}[0.21]{\includegraphics{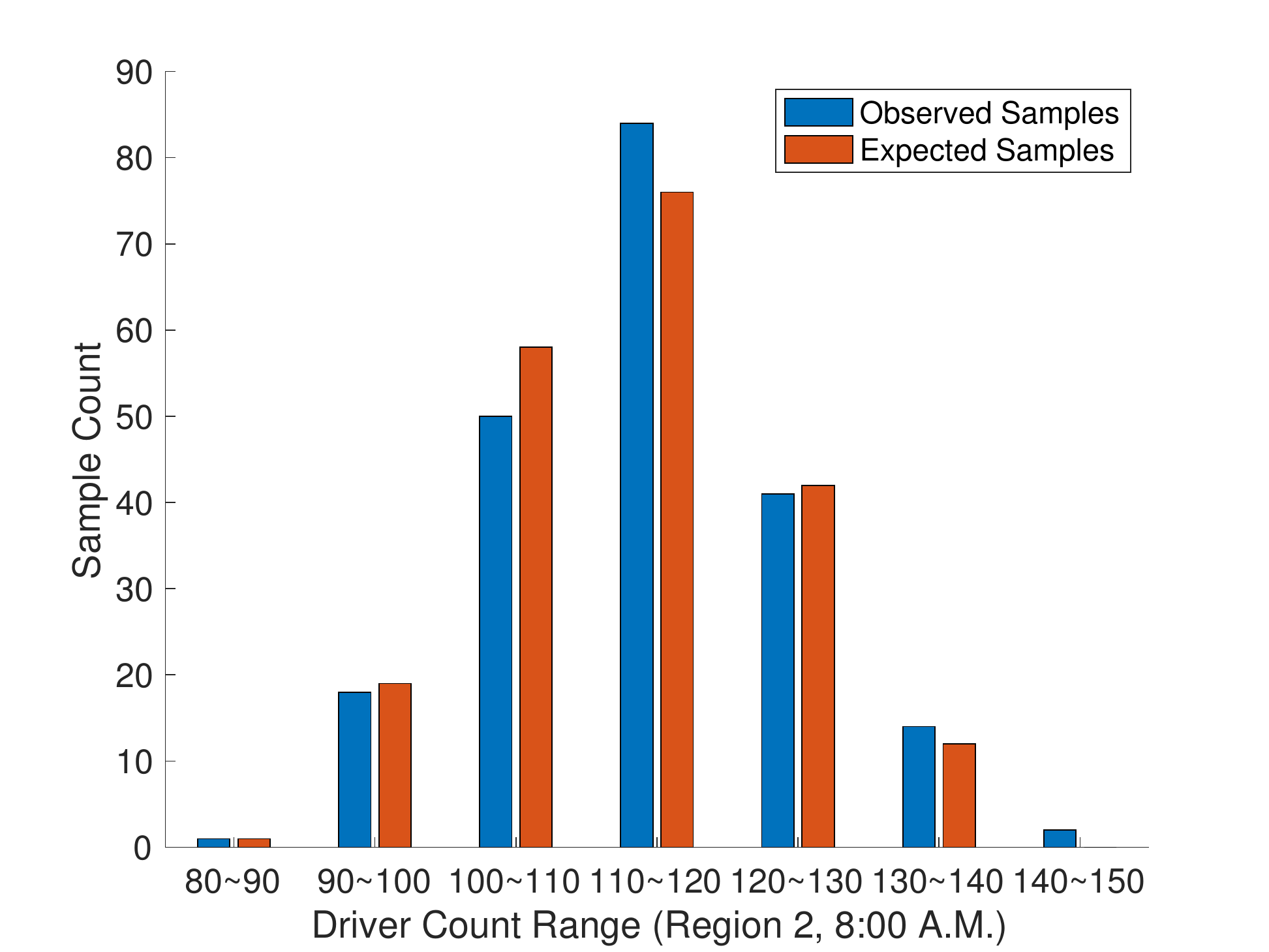}}
		\label{subfig:driver_r2_8_bar}}\vspace{-1ex}
	\caption{\small Distribution of Driver Quantity Samples in Different Regions and Time Periods.}
	\label{fig:driver_poisson} 
\end{figure*}

We use the first five months to train the demand count prediction model, and use the last month to test the performance of our dispatching algorithm. During prediction, we use 30 minutes as a time slot. Table \ref{tab:tab_exp} lists the statistic information of the taxi trip record data set.


\begin{table}[h!]
	\centering
	\caption{The statistics of the taxi trip dataset}\label{tab:tab_exp} 
	\begin{tabular}{c|c}
		& Yellow Taxi \\\hline\hline
		Days of Train Records  & 91\\
		Days of Evaluation Records    & 20\\
		Days of Test Records   & 10\\
		max Records Count Per  Slot & 853 \\
		min  Records Count Per  Slot & 0 \\
		max  Records Count Per  Zone &12017 \\
		min Records Count Per  Zone &0 \\
		\hline
	\end{tabular}
\end{table}

To evaluate the accuracy of our tested prediction models, we report the RMSE loss as the evaluation metric. To show the effectiveness of the model, we compare our model with the following prediction models:
\begin{itemize}
	\item HA. Historical Average calculates the mean of the order records in the previous 15 time slots  as the next order count.
	
	\item LR. Linear Regression model collects the order records in the previous 15 time slots to predict the next order count.
	
	\item GBRT. Gradient Based Regression Tree  \cite{friedman2002stochastic} model  collects the order records in the previous 15 time slots and uses non-parametric regression to predict the next  order count.
\end{itemize}

The evaluation results are shown on the table  \ref{tab:tab_eval}, we can see that the DeepST model performs best. Thus, we only use DeepST model in our experiments.

\begin{table}[h]
	\centering
	\caption{Results of the Demand Prediction Methods}\label{tab:tab_eval} 
	\begin{tabular}{c|c|c}
		& RMSE (\%) & Real RMSE (s) \\\hline\hline
		DeepST & 2.30 & 15.03 \\
		HA   & 7.46 & 48.21 \\
		LR  & 3.40 & 21.66 \\
		GBRT & 2.74 & 17.67 \\
		\hline
	\end{tabular}
\end{table}

Sometimes, the whole space is not divided into regular rectangle grids (e.g., New York City has its 262 irregular taxi zones), then  we cannot use Convolutional Neural Network directly to extract the adjacent information. To solve this issue, we replace the conv layer with Graph Convolutional layer, which comes from Graph Convolutional Network \cite{thomas2016semi}. We name the modified model as DeepST-GC (DeepST with Graph Convolutional Network). Next, we introduce the basic component of Graph Convolution Layer of DeepST-GC.

We represent the whole New York City as a graph $G=\{V,E\}$, where each taxi zone is represented as a vertex and their connectivities is regarded as edges. All nodes' features can be stored in a matrix $X$  whose dimension is $[N, F]$, $N$ means the number of taxi-zone, $F$ indicates the number of features of each taxi-zone. Connectives matrix, $\tilde{A}$, is added with an identity matrix, $I$, to form the adjacency matrix. The adjacency matrix helps to pass the features of nodes to their related nodes. Before we use the adjacency matrix, we normalize the matrix such that all the summations of each row is equal to 1. Let $D$ be the diagonal node degree matrix, then the normalized adjacency matrix would be:
\begin{equation}
A=D^{-\frac{1}{2}}(\tilde{A}+I)D^{-\frac{1}{2}}\notag
\end{equation}
The whole graph convolution computation is the formulation:
\begin{equation}
X^{t+1}=f(X^t,A)=\sigma(A^TX^tW^t),\notag
\end{equation}
\noindent where $X^t$ and $W^t$ are the input and weights of the $t$th layer of the neural network.

This modified DeepST-GC model can support the situations where the whole space is not evenly divided.

\section{Distribution of Data Set}

In this section, we test the validity of the assumption that the numbers of new orders and rejoined drivers in a given region obey Poisson distributions. We use chi-square ($\chi^2$) test \cite{greenwood1996guide}, a well-known and commonly used statistical hypothesis test, to verify the distribution of orders at different times in two example regions. Chi-square test can be used to test whether two variables are related or independent from one another or to test the goodness-of-fit between an observed distribution (i.e., the observed data) and a theoretical distribution of frequencies (i.e., the hypothesis). 

In this experiment, we tested the order quantity 
distribution in region 1 (\ang{-73.97} $\sim$ \ang{-74.01}, \ang{40.70} $\sim$ \ang{40.80}), and region 2 (\ang{-73.93} $\sim$ \ang{-73.97}, \ang{40.70} $\sim$ \ang{40.80}) in New York of at 7 A.M. and 8 A.M. respectively.

We assume that orders follow the same distribution over a short period (10 minutes), and set the random variable $X$ as the number of orders per minute. Then, for each region at each time, we use the number of orders for each minute as a sample.  Since there are 21 working days in January 2013, there are 210 samples, which are denoted as $X_{i}$ ($i$ = 1, 2, 3, 4..., 210). Our hypothesis is as follows:
$$
H:X\sim P(\alpha)
$$

We take the appropriate constant $a_i$ ($i$ = 1, 2, 3, ..., $r-1$), 
and decompose $(-\infty, +\infty)$ into several intervals, which are denoted $I_i$ ($i$ = 1, 2, 3, ..., $r$). 
Then we count the number of samples that belong to the interval $I_i$, and $\nu_{i}$ is called the observed frequency of $I_i$. Let $p_{i}$ be the probability 
that the value of the random variable $X$ belongs to the interval $I_{i}$ while the hypothesis holding, and the theoretical frequency of $I_{i}$ can be denoted as $np_{i}$. We define the statistic $k$ as follows:
$$
k=k\left(X_1,X_2,...,X_n;P\right)=\sum_{i=1}^{r}{\frac{\left(\nu_i-np_i\right)^2}{np_i}}
$$

\begin{table}[t!]
	\centering 
	{
		\caption{\small $\chi^2$ Test Results of Orders} \label{tab:chi_order}
		\begin{tabular}{ll|lll}
			region&time slot&$r$&$k$&$\chi_{r-1}^2\left(0.05\right)$ \\ \hline \hline
			region 1&7:00$\sim$7:10&7&8.7474&12.592\\
			region 1&8:00$\sim$8:10&7&6.3022&12.592\\
			region 2&7:00$\sim$7:10&6&7.2330&11.070\\
			region 2&8:00$\sim$8:10&5&7.7089&9.488\\
			\hline
			\hline
		\end{tabular}
	}
\end{table}

\begin{table}[t!]
	\centering
	{
		\caption{\small $\chi^2$ Test Results of Drivers}\label{tab:chi_driver}
		\begin{tabular}{ll|lll}
			region&time slot&$r$&$k$&$\chi_{r-1}^2\left(0.05\right)$ \\ \hline \hline
			region 1&7:00$\sim$7:10&6&7.7964&11.070\\
			region 1&8:00$\sim$8:10&7&8.8335&12.592\\
			region 2&7:00$\sim$7:10&5&8.8526&9.488\\
			region 2&8:00$\sim$8:10&5&6.7923&9.488\\
			\hline
			\hline
		\end{tabular}
	}
\end{table}

K. Pearson \cite{pearson1900x} proved that if $H$ is assumed to be true, then when the sample size is $\infty$, the distribution of statistic $k$ converges to $\chi_{r-1}^2$, that is, the distribution of $k$ is chi-square distribution with the degree of freedom being $r-1$. In this case that we set the confidence as $\beta$, we can decide whether to deny $H$ by judging whether $k>\chi_{r-1}^2\left(\beta\right)$ is true.

As shown in Table \ref{tab:chi_order}, the value of $k$ is smaller than $\chi_{r-1}^2\left(0.05\right)$, which indicates that we cannot reject the null hypothesis with the confidence coefficient of 0.05 (i.e., the null hypothesis is significant; the number of orders follow a Poisson distribution in a high probability).

To clearly illustrate that the observed distribution is close to the expected distribution, Figure \ref{fig:order_poisson} shows the expected and observed distribution of order quantity samples in different regions and periods obtained in our experiment. The observed order distribution  fits well with the expected distribution, which verify that the order quantity obeys the Poisson distribution.

In call-hailing platforms (e.g., DiDi Chuxing), regular drivers (i.e., the ones work more than 5 days a month) usually work for more than 8 hours. Then, the destinations of orders can be considered as the rejoined drivers' birth-location. We also conduct the chi-square test on the destinations of orders (the location of rejoined drivers). The results are shown in Table \ref{tab:chi_driver}. We can achieve a same conclusion: the rejoined drivers also follow a Poisson distribution. 
Figure \ref{fig:driver_poisson} shows the expected and observed distribution of driver quantity samples in different regions and periods obtained in our experiment. The similarity between observed and expected driver distributions  verifies that the driver quantity also is a Poisson distribution.

\section{Results of maximizing the number of total served orders}
\label{sec:max_order_count}
Our queueing-theoretic framework can also support the optimization goal of maximizing the number of total served orders. We can simply modify our idle ratio greedy algorithm to handle the new goal of maximizing the number of total served orders. Specifically, we modify IRG to greedily select the order with the minimum summation of travel cost and idle time. We call the modified algorithm shortest total time greedy algorithm (SHORT). Then, we conduct a new set of experiments to test the efficiency of our SHORT algorithm. We still use the setting of experiments in Section \ref{sec:experimental}. 

\begin{figure}[h!]\centering\vspace{-2ex}
		\subfigure{
		\scalebox{0.12}[0.12]{\includegraphics{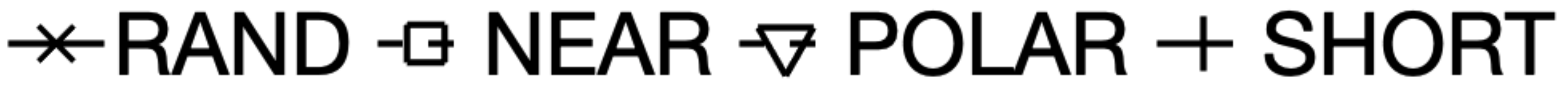}}}\hfill\\
	\addtocounter{subfigure}{-1}\vspace{-2ex}
	\subfigure[][{\small Effect of Number of Drivers (n)}]{
		\scalebox{0.2}[0.2]{\includegraphics{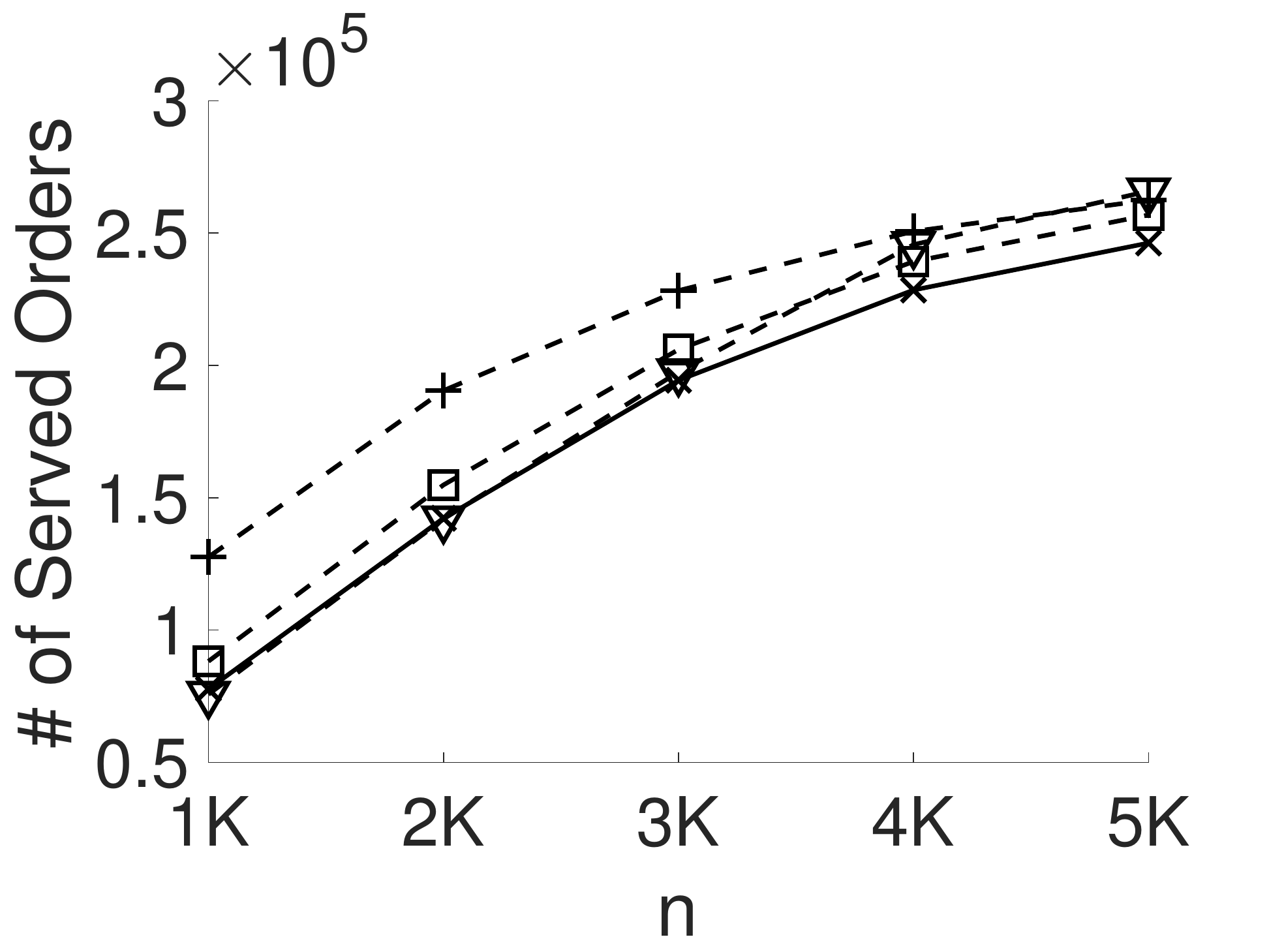}}
		\label{subfig:driver_c}}
	\subfigure[][{\small Effect of Time Window ($t_c$)}]{
		\scalebox{0.2}[0.2]{\includegraphics{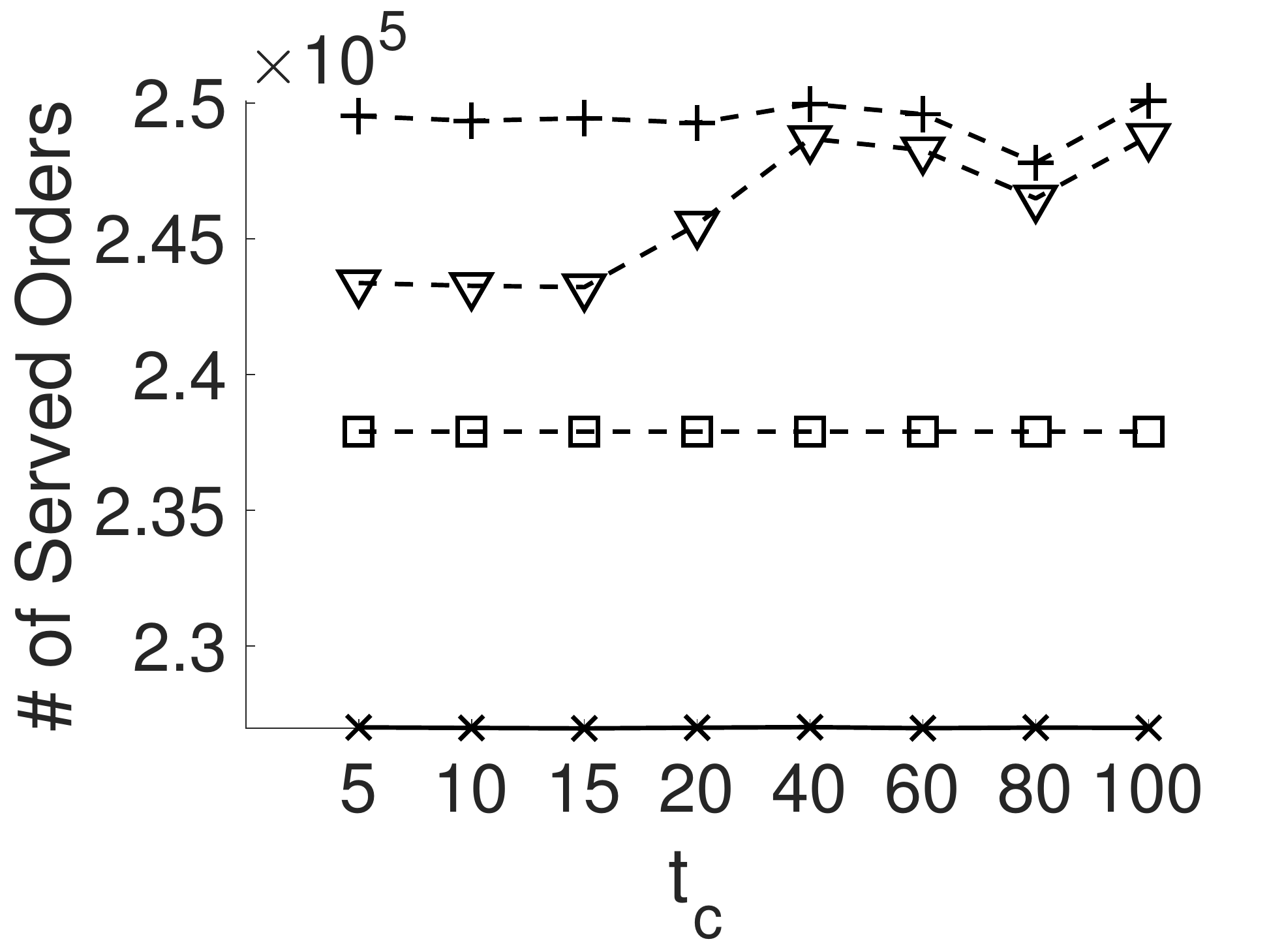}}
		\label{subfig:time_c}}\vspace{-2ex}
	\subfigure[][{\small Effect of Batch Length ($\Delta$)}]{
		\scalebox{0.2}[0.2]{\includegraphics{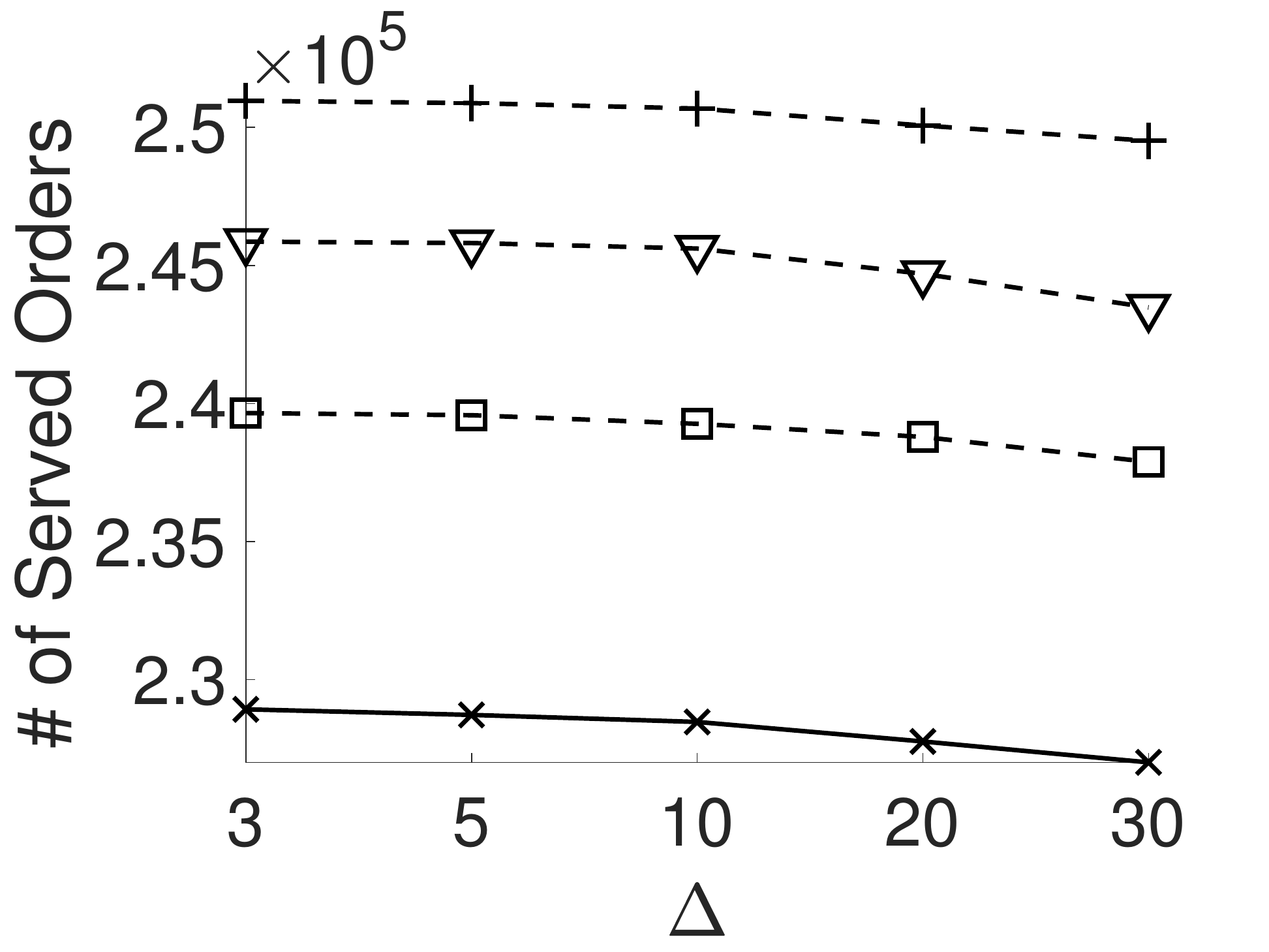}}
		\label{subfig:batch_c}}\vspace{-2ex}
	\subfigure[][{\small Effect of Base Waiting Time ($\tau$)}]{
		\scalebox{0.2}[0.2]{\includegraphics{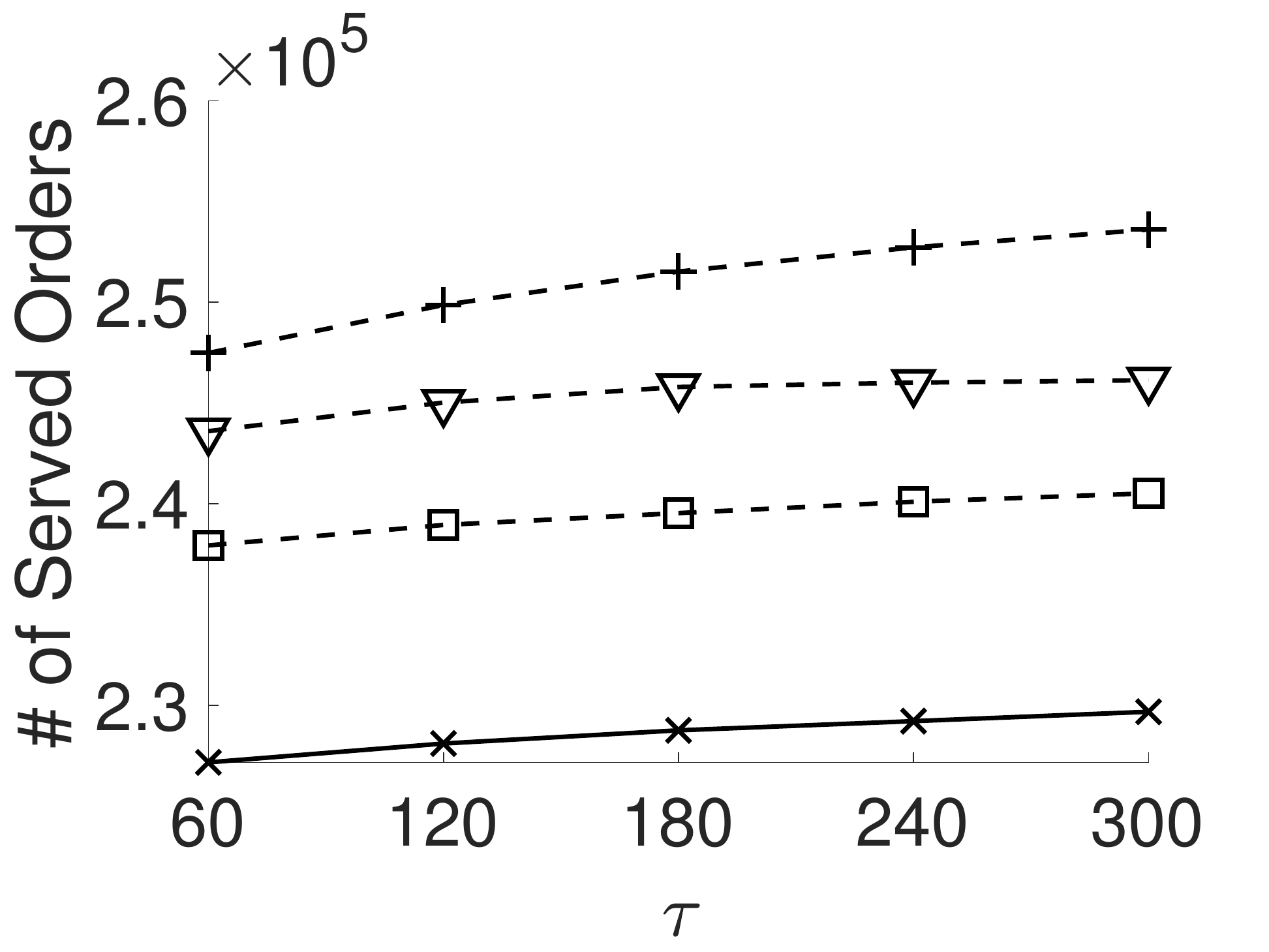}}
		\label{subfig:wait_c}}\vspace{-2ex}
	\caption{\small Results of Total Served Orders.}\vspace{-2ex}
	\label{fig:order_count} 
\end{figure}

To clearly show the results, we only report the results of RAND, NEAR, POLAR and SHORT on the number of total served orders in Figure \ref{fig:order_count}. As shown in Figure \ref{subfig:driver_c}, when the number of drivers increases, the numbers of total served orders for all the tested algorithms all increases. Our SHORT can achieve the highest number of served orders. NEAR is better than POLAR when $n$ is 1K$\sim$3K. When $n$ increases to 4K$\sim$5K, POLAR can assign more orders than NEAR. Figure \ref{subfig:time_c} shows the results of tested approaches on varying the time window of estimating the arrival rate of new riders and rejoined drivers. Our SHORT can always finish more orders than other tested approaches. POLAR is better than NEAR and RAND. Similar situations happen in the results of varying the batch length (in Figure \ref{subfig:batch_c}) and the base waiting time (in Figure \ref{subfig:wait_c}).

\end{document}